\documentclass{article}
\usepackage{amsmath, amssymb, amscd, amsthm}
\usepackage{graphicx}
\usepackage{hyperref}
\usepackage{makecell}
\usepackage[utf8]{inputenc} 
\usepackage[T1]{fontenc}    
\usepackage{hyperref}      
\usepackage{multirow}
\usepackage{url}            
\usepackage{booktabs}       
\usepackage{amsfonts}      
\usepackage{nicefrac}       
\usepackage{microtype}      
\usepackage[dvipsnames]{xcolor}
\usepackage{subcaption}
\usepackage{bm}
\usepackage{algorithm}
\usepackage{algorithmic}
\usepackage{hhline}
\allowdisplaybreaks

\usepackage[capitalize,noabbrev]{cleveref}
\theoremstyle{plain}

\newtheorem{theorem}{Theorem}
\newtheorem{lemma}{Lemma}
\newtheorem{proposition}{Proposition}
\newtheorem{corollary}{Corollary}

\theoremstyle{definition}
\newtheorem{definition}{Definition}
\newtheorem{question}{Question}
\newtheorem{remark}{Remark}
\newtheorem{assumption}{Assumption}

\DeclareMathOperator*{\argmin}{arg\,min}
\DeclareMathOperator*{\argmax}{arg\,max}
\DeclareMathOperator*{\id}{Id}
\newcommand{\BCM}{\normalfont{\text{$W_{2}$BCM}}}
\newcommand{\expe}{\exp_\varepsilon}
\newcommand{\LBCM}{\normalfont{\text{LBCM}}}
\newcommand{\norm}[1]{\left\lVert#1\right\rVert}

\makeatletter
\def\thanks#1{\protected@xdef\@thanks{\@thanks
        \protect\footnotetext{#1}}}
\makeatother

\title{Linearized Wasserstein Barycenters: Synthesis, Analysis, Representational Capacity, and Applications}
\author{Matthew Werenski$^{(\dagger) }$, Brendan Mallery$^{(\ddagger)}$, Shuchin Aeron$^{(\square, *)}$, James M. Murphy$^{(\ddagger,*)}$ 
  \thanks{\hspace{-16pt}\text{($\dagger$): Department of Computer Science, Tufts University}\\ \text{$(\ddagger)$: Department of Mathematics, Tufts University}\\ \text{$(\square)$: Department of Electrical and Computer Engineering, Tufts University.}\\
  \text{($*$): Co-last authors}\\
  Correspondence to: Matthew.Werenski@tufts.edu} 
  } 
\begin{document}
\maketitle
\begin{abstract}

We propose the linear barycentric coding model (LBCM) which utilizes the linear optimal transport (LOT) metric for analysis and synthesis of probability measures.  We provide a closed-form solution to the variational problem characterizing the probability measures in the LBCM and establish equivalence of the LBCM to the set of 2-Wasserstein barycenters in the special case of compatible measures. Computational methods for synthesizing and analyzing measures in the LBCM are developed with finite sample guarantees. One of our main theoretical contributions is to identify an LBCM, expressed in terms of a simple family, which is sufficient to express all probability measures on the closed unit interval.  We show that a natural analogous construction of an LBCM in 2 dimensions fails, and we leave it as an open problem to identify the proper extension in more than 1 dimension.  We conclude by demonstrating the utility of LBCM for covariance estimation and data imputation.
 
\end{abstract}

\section{Introduction}

Interpreting data as probability measures and embedding them into spaces equipped with an appropriate metric (or more generally, discrepancy) to allow for meaningful comparisons has emerged as an important paradigm in computer vision \cite{bonneel2023survey}, natural language processing \cite{xu2018distilled}, high-energy physics \cite{komiske2019metric, cai2022metric}, time-series analysis \cite{cheng2021dynamical, cheng2023non}, and seismic imaging \cite{metivier2016optimal} among other areas.  In the space of probability measures, one can utilize the discrepancy used for comparisons to develop low-parameter and parsimonious models for data. In this paper, we consider a class of models that we term \textit{barycentric coding models (BCM)} \cite{werenski2022measure}.  These models allow for \textit{synthesis} and \textit{analysis} of measures through variational problems with respect to a given set of \textit{reference measures}, $\{\mu_i\}_{i=1}^{m}$, and non-negative coefficients collected in the vector $\lambda$ belonging to the $(m-1)$-dimensional simplex $\Delta^m$.  Synthesis takes the form:
\begin{equation}\label{eqn:General_Synth}
\mu_{\lambda}\triangleq\argmin_{\nu\in\mathcal{P}(\mathbb{R}^{d})}\sum_{i=1}^{m}\lambda_{i}D_{s}(\nu,\mu_{i})+\mathcal{R}(\nu), 
 \end{equation}where $D_{s}$ is a discrepancy and $\mathcal{R}$ is a regularizer enforcing structural properties on $\mu_{\lambda}$.  We call the collection of measures generated as in (\ref{eqn:General_Synth}) a BCM.  One can also project a measure $\eta$ onto the collection of measures of the form (\ref{eqn:General_Synth}) via solving the \textit{analysis problem} 
\begin{equation}\label{eqn:Analysis}
\lambda_{\eta}=\argmin_{\lambda \in \Delta^m} D_{a}(\eta, \mu_{\lambda}), 
 \end{equation}
 where $D_{a}$ is a discrepancy, not necessarily the same as $D_{s}$.  We consider $D_{s}, D_{a}$ that are based on optimal transport \cite{santambrogio2015optimal, peyre2020computational} to impute the geometry of the space of probability measures onto our low-dimensional models.

Specifically, we consider the case when $D_{s}=D_{a}$ are given by a \textit{linear optimal transport (LOT) metric} \cite{wang2013linear}.  Given an absolutely continuous \textit{base measure} $\mu_0$, the (squared) LOT metric between measures $\mu,\nu$ is
\begin{align*}
\|T_{\mu_{0}}^{\mu}-T_{\mu_{0}}^{\nu} \|_{L^{2}(\mu_{0})}^{2}=\int\|T_{\mu_{0}}^{\mu}(x)-T_{\mu_{0}}^{\nu}(x)\|_2^{2}d\mu_{0}(x),
\end{align*}
where for any measure $\eta$, $T_{\mu_0}^{\eta}$ is the optimal transport map from $\mu_0$ to $\eta$ with respect to the squared-Euclidean ground cost (see Section \ref{sec:Background}). We call the set of all measures that can be represented this way the \textit{linear barycentric coding model (LBCM):}

\begin{definition}
   Let $\{\mu_i\}_{i=1}^{m}$ be reference measures and $\mu_{0}$ a base measure.  The associated \textit{linear barycentric coding model (LBCM)} is
\begin{align}
    \LBCM(\{\mu_{i}\}_{i=1}^{m}, \mu_0)
    &\triangleq\{\mu_{\lambda} \ | \ \lambda\in\Delta^{m}\},\notag\\
    \mu_{\lambda} &\triangleq \argmin_{\nu\in\mathcal{P}(\mathbb{R}^{d})}\sum_{i=1}^{m} \lambda_{i} \|T_{\mu_{0}}^{\nu}-T_{\mu_{0}}^{\mu_i} \|_{L^{2}(\mu_{0})}^{2} \label{eqn:LBCM_VariationalProblem}.
    \end{align}
\end{definition}

\subsection{Summary of Contributions}

This paper focuses on modeling data with the LBCM.
\begin{enumerate}
    \item Proposition \ref{prop:LBCM_VariationalForm} establishes that solutions to (\ref{eqn:LBCM_VariationalProblem}) are of the form:
\begin{equation*}
   \mu_{\lambda} = \left(\sum_{i=1}^{m} \lambda_{i}T_{\mu_{0}}^{\mu_{i}} \right )\#\mu_0.
\end{equation*}Proposition \ref{thm:LBCM_Equiv_BCM} proves that if the measures $\{\mu_{0}\}_{i=1}^{m}$ satisfy a particular compatibility condition, the LBCM is equivalent to the BCM with $D_{s}=D_{a}= W_2^{2}$ the squared 2-Wasserstein metric \cite{werenski2022measure}.

\item We show that the analysis problem, 
\begin{equation*}
    \lambda_{\eta}=\argmin_{\lambda \in \Delta^m} \left\|\sum_{i=1}^m \lambda_i T_{\mu_{0}}^{\mu_{i}} - T_{\mu_{0}}^{\eta}\right\|_{L^2(\mu_0)}^2,
\end{equation*}
can be efficiently solved via a quadratic program when $\eta\in\LBCM(\{\mu_{i}\}_{i=1}^{m},\mu_{0})$. Leveraging recent results in entropy-regularized optimal transport, we establish rates of statistical estimation when the measures are available via samples. 

\item  In Section \ref{sec:RepresentationalCapacity}, we propose a novel problem of representational capacity in the space of probability measures.  In Theorem \ref{thm : lbcm 1D capacity} we show that the set of reference measures supported on $\{a,b\}$ with $\mu_{0}$ as the uniform measure on $[a,b]$ has dense LBCM in $\mathcal{P}([a,b])$ with respect to weak convergence. We then establish in Theorem \ref{thm : questions negative} that a natural generalization of our 1-dimensional result fails in dimension 2, due to the increased complexity of the space of transport maps.

\item We provide experiments to show the utility of the LBCM model for covariance estimation and image reconstruction.  We give both qualitative and quantitative comparisons with a $W_{2}$ variant of the BCM and linear mixtures, and show that reconstruction with the LBCM gives comparable results while being significantly faster than other methods. We also investigate how changing the base measure affects the quality of reconstruction.
\end{enumerate}

All proofs are deferred to the Appendix.

\section{Background}
\label{sec:Background}
    
For a vector $x\in\mathbb{R}^{d}$ (or random variable or function), we let $x_{i}$ denote its $i^{th}$ coordinate.  For $t\in\mathbb{R}$, $\expe(t)\triangleq\exp(t/\varepsilon)$.  The set of probability measures in $\mathbb{R}^d$ will be denoted by $\mathcal{P}(\mathbb{R}^d)$. The set of measures which are absolutely continuous with respect to the Lebesgue measure on $\mathbb{R}^d$ will be denoted $\mathcal{P}_{ac}(\mathbb{R}^d)$. For $\Omega \subseteq \mathbb{R}^d$ we let $\mathcal{P}(\Omega),\mathcal{P}_{ac}(\Omega)$ respectively be the set of measures $\mu$ in $\mathcal{P}(\mathbb{R}^d), \mathcal{P}_{ac}(\mathbb{R}^d)$ such that $\mu[\Omega] = 1.$  Elements of $\mathcal{P}(\Omega)$ and $\mathcal{P}_{ac}(\Omega)$ with finite second moment are denoted $\mathcal{P}_{2}(\mathbb{R}^d)$ and $\mathcal{P}_{2,ac}(\mathbb{R}^d)$, respectively.  For a convex body $C$ in $\mathbb{R}^d$ (i.e., $C$ is convex, compact, and has a non-empty interior) we will denote the uniform measure over $C$ by $U_C$.
Note that we always have $U_C \in \mathcal{P}_{ac}(C)$. 
For any collection of elements $\{v_{i}\}_{i\in I}$ of a vector space with real scalars, we let $\normalfont{\text{conv}}(\{v_{i}\}_{i\in I})$ denote the \textit{convex hull} consisting of all convex combination (i.e., linear mixtures) of elements of $\{v_{i}\}_{i\in I}$.  We use $\lesssim$ to specify inequalities which hold up to multiplicative constants.

\begin{definition} 
    Let $\mu\in\mathcal{P}_{2,ac}(\mathbb{R}^{d})$ and $\nu\in\mathcal{P}_{2}(\mathbb{R}^{d})$.  The \textit{2-Wasserstein metric} between $\mu$ and $\nu$ is 
    \begin{equation}
    \label{eqn:W2}W_{2}^2(\mu,\nu)\triangleq\min_{T\#\mu=\nu} \int_{\mathbb{R}^{d}}\|T(x)-x\|_{2}^{2}d\mu(x).
    \end{equation}We call the $T_{\mu}^{\nu}$ realizing (\ref{eqn:W2}) the 2-Wasserstein \textit{optimal transport map} from $\mu$ to $\nu$. 
    \end{definition}

The use of a minimum rather than an infimum in (\ref{eqn:W2}) is justified in our setting by the following seminal result \cite{brenier1991polar, gangbo1996geometry, mccann1997convexity}.  

\begin{theorem}[Brenier]\label{thm : knott-smith and brenier}Let $\mu \in \mathcal{P}_{2,ac}(\mathbb{R}^{d})$ and $\nu \in \mathcal{P}_{2}(\mathbb{R}^{d})$. Then there exists a unique (up to $\mu$-a.e. equivalence) map $T^{*}$ such that $T^{*} \# \mu = \nu$, and
\begin{equation*}
    \min_{T \# \mu = \nu} \int_{\mathbb{R}^{d}} ||T(x) - x||_2^2 d\mu(x)=\int_{\mathbb{R}^{d}} ||T^{*}(x) - x||_2^2 d\mu(x).
\end{equation*}Moreover, there exists a convex function $\varphi$ such that $\nabla\varphi=T^{*}$.
\end{theorem}

Note, one can extend $W_{2}(\mu,\nu)$ to $\mu$ that are not absolutely continuous by optimizing over couplings rather than maps \cite{santambrogio2015optimal}.

\subsection{Related Work}
\label{subsec:RelatedWork}

Many different barycentric models of the form (\ref{eqn:General_Synth}) have been considered for synthesis; see Table \ref{tab:RelatedWork}.  On the other hand, the analysis problem has received relatively little attention \cite{bonneel2016wasserstein, werenski2022measure, blickhan2024registration, gunsilius2024tangential}.  The most closely related work to ours is \cite{werenski2022measure} where one works with the BCM model with $D_{s}=D_{a}=W_{2}^{2}$, which we refer to as the $\BCM$.
\begin{definition}
   Let $\{\mu_i\}_{i=1}^{m}$ be reference measures.  The associated \textit{$W_{2}$-barycentric coding model ($W_{2}$BCM)} is
\begin{align}\label{eqn:BCM}
    \BCM(\{\mu_{i}\}_{i=1}^{m})
    &\triangleq\{\mu_{\lambda} \ | \ \lambda\in\Delta^{m}\},\\
    \mu_{\lambda}&\triangleq\argmin_{\nu\in\mathcal{P}(\mathbb{R}^{d})}\sum_{i=1}^{m} \lambda_{i}W_{2}^{2}(\nu,\mu_{i}).\notag
    \end{align}
\end{definition}
Under some regularity assumptions on the reference measures and $\mu$, determining the coordinates $\lambda\in\Delta^{m}$ associated with $\mu\in \BCM(\{\mu_{i}\}_{i=1}^{m})$ can be done via a quadratic program whose parameters depend on $\mu$ and $\{\mu_{i}\}_{i=1}^{m}$ (see Theorem 1 in \cite{werenski2022measure}).  When measures are observed via samples, plug-in estimates of OT maps can be used for the QP formulation with guarantees on error rates  (see Corollary 2 in \cite{werenski2022measure}).  We prove similar results for the LBCM in Section \ref{sec:SynthesisAnalysis}.

In the context of high computational complexity and curse of dimensionality in statistical error rates for estimating quantities related to $W_2$, \cite{mallery2025synthesis} considers entropy-regularized optimal transport for analysis and synthesis of measures.

We note other parsimonious models in the space of probability measures are possible including geodesic PCA \cite{bigot2017geodesic, cazelles2018geodesic} and log-PCA \cite{fletcher2004principal, sommer2010manifold}.  In this paper we are concerned with analysis and synthesis given the LBCM model, but the problem of identifying a BCM model given data assumed to be generated by the BCM have also been considered \cite{schmitz2018wasserstein, cheng2021dynamical, mueller2023geometrically, cheng2023non}.

\begin{table*}[ht!]
\begin{center}\resizebox{\textwidth}{!}{
\begin{tabular}{| c || c | c | c | } 
\hline
Barycenter & $D(\nu,\mu_{i})$ (Discrepancy) & $\mathcal{R}(\nu)$ (Regularizer) \\

\hhline{|=||=|=|=|}
Linear & $\displaystyle\int\bigg|\frac{d\nu}{dx}(x)-\frac{d\mu_{i}}{dx}(x)\bigg|^2dx$ & none \\
\hline
2-Wasserstein \cite{agueh2011barycenters} & $W_{2}^{2}(\nu,\mu_{i})$ & none \\
\hline
Inner-Regularized Entropic \cite{janati2020debiased, mallery2025synthesis} & $S_{\varepsilon}(\nu,\mu_{i})$ & none \\
\hline
Outer-Regularized Entropic \cite{bigot2019penalization, carlier2021entropic} & $W_{2}^{2}(\nu,\mu_{i})$ & $\displaystyle\int \log\left(\frac{d\nu}{dx}(x)\right)d\nu(x)$ \\
\hline
Doubly-Regularized \cite{chizat2023doubly, vaskevicius2024computational} & $S_{\varepsilon}(\nu,\mu_{i})$ & $\displaystyle\int \log\left(\frac{d\nu}{dx}(x)\right)d\nu(x)$ \\
\hline
Sinkhorn \cite{janati2020debiased, mallery2025synthesis} & $S_{\varepsilon}(\nu,\mu_{i})-\frac{1}{2}S_{\varepsilon}(\nu,\nu)-\frac{1}{2}S_{\varepsilon}(\mu_{i},\mu_{i})$ & none \\
\hline
Linearized 2-Wasserstein (\cite{merigot2020quantitative}, this paper) & $\int \|T_{\mu_{0}}^{\nu}(x)-T_{\mu_{0}}^{\mu_{i}}(x)\|_2^{2}d\mu_{0}(x)$ & none\\
\hline

\end{tabular}}
\caption{ \label{tab:RelatedWork}Summary of methods to synthesize probability measures in the framework of (\ref{eqn:General_Synth}).}
\label{tab:survey_synth}
\end{center}
\vspace*{-\baselineskip}
\end{table*}

The use of transport maps in the LBCM is motivated by \textit{linearized optimal transport}, in which for a chosen base measure $\mu_{0}$, one approximates $W_{2}(\mu,\nu)\approx\|T_{\mu_{0}}^{\mu}-T_{\mu_{0}}^{\nu}\|_{L^{2}(\mu_{0})}$ \cite{wang2013linear}.  The embedding of measures $\mu,\nu$ into $L^{2}(\mu_{0})$ via $\mu\mapsto T_{\mu_{0}}^{\mu}$ is not an isometry in general. However, upper and lower bounds between $W_{2}(\mu,\nu)$ and $\|T_{\mu_{0}}^{\mu}-T_{\mu_{0}}^{\nu}\|_{L^{2}(\mu_{0})}$ exist under suitable regularity conditions on the measures considered \cite{delalande2023quantitative}. 
 Conditions guaranteeing an exact isometry are closely related to the notion of compatible measures described in Definition \ref{defn:compatible} of \cite{moosmuller2023linear}.  

The recent paper \cite{jiang2023algorithms} proposes the use of LBCM sets as a regularizer for variational inference.  Therein is proved a quantitative approximation result for the number of reference measures required to estimate a sufficiently regular optimal transport map in a 1-dimensional LBCM.  Our results in Section \ref{sec:RepresentationalCapacity} may be compared to this, albeit our result does not require smoothness of the transport maps and shows that the reference measures can be taken to have singular support. Moreover, we investigate the multidimensional setting and identify interesting problems therein.

\section{Properties of the LBCM}

\label{sec:BasicProperties}

The following result asserting that the variational form of the LBCM has a closed form is stated informally in \cite{merigot2020quantitative}; we give a proof in Appendix \ref{sec:ProofsSection3} for completeness.
\begin{proposition}
\label{prop:LBCM_VariationalForm}For $\mu_{0},\{\mu_{i}\}_{i=1}^{m}\subset\mathcal{P}(\mathbb{R}^{d})$ and $\lambda\in\Delta^{m},$ the optimization problem 
\begin{equation}\label{eqn:LBCM-Barycentric-Formulation}
\argmin_{\nu\in\mathcal{P}(\mathbb{R}^{d})}\sum_{i=1}^{m}\lambda_{i}\|T_{\mu_{0}}^{\nu}-T_{\mu_{0}}^{\mu_{i}}\|_{L^{2}(\mu_{0})}^{2}
\end{equation}has solution $\left( \sum_{i=1}^m \lambda_i T_{\mu_{0}}^{\mu_{i}} \right )\#\mu_0$.  
    
\end{proposition}

Proposition \ref{prop:LBCM_VariationalForm} demonstrates that the LBCM is similar in spirit to generative model for data associated to nonnegative matrix factorization (NMF) \cite{lee2000algorithms} and archetypal analysis \cite{cutler1994archetypal}, in which observations are modeled as non-negative (or convex) combinations of generators. In the case of the LBCM, the convex combination of maps are then used to pushforward the base measure $\mu_{0}$.

While they differ in general, the LBCM is equal to the $\BCM$ for a common set of reference measures in the special case in which the underlying measures are \textit{compatible} \cite{panaretos2020invitation}. 

\begin{definition}\label{defn:compatible}
    Measures $\{\nu_{i}\}_{i=1}^m\subset\mathcal{P}(\mathbb{R}^{d})$ are \textit{compatible} if for any $\nu_i,\nu_j,\nu_k$ the optimal transport maps satisfy $T_{\nu_{i}}^{\nu_{j}} = T_{\nu_{k}}^{\nu_{j}}\circ T_{\nu_{i}}^{\nu_{k}}$.
\end{definition}

Compatibility of measures is a strong condition, but does hold in several important cases, including: (i) all absolutely continuous 1-dimensional probability measures; (ii) Gaussian measures with simultaneously diagonalizable covariance matrices ; (iii) measures that are translation-dilations of a common measure. We refer to \cite{panaretos2020invitation} for further discussion of compatibility. The following equivalence result follows from several known results on multimarginal optimal transport; a proof is given in Appendix \ref{sec:ProofsSection3} for completeness.

\begin{proposition}
\label{thm:LBCM_Equiv_BCM}
    Suppose $\{\mu_{i}\}_{i=0}^{m}$ are compatible.  Then 
    \begin{align*}\LBCM(\mu_{0};\{\mu_{i}\}_{i=1}^{m})=\BCM(\{\mu_{i}\}_{i=1}^{m}),
    \end{align*}
    and $\nu_\lambda$ is the same in both sets.
\end{proposition}

\section{Synthesis and Analysis in the LBCM}
\label{sec:SynthesisAnalysis}
Fix base and reference measures $\{\mu_{i}\}_{i=0}^{m}$. The \textit{synthesis problem} in the LBCM is, via Proposition \ref{prop:LBCM_VariationalForm}, to compute $\left(\sum_{i=1}^{m}T_{\mu_{0}}^{\mu_{i}}\right)\#\mu_{0}$. 
 For a new measure $\eta$ the \textit{analysis problem} in the LBCM is to find coefficients that optimally represent $T_{\mu_{0}}^{\eta}$ as a convex combination of the maps $\{T_{\mu_{0}}^{\mu_{i}}\}_{i=1}^{m}$.  Specifically, it solves the optimization:
\begin{equation*}
    \min_{\lambda \in \Delta^m} \big\|\sum_{i=1}^m \lambda_i T_{\mu_{0}}^{\mu_{i}} - T_{\mu_{0}}^{\eta}\big\|_{L^2(\mu_0)}^2.
\end{equation*}
One can equivalently write this as
\begin{equation} \label{eqn:AnalysisLBCM}
    \min_{\lambda\in\Delta^{m}}\lambda^{T}A^{L}\lambda,
\end{equation}where $A^L_{ij} = \displaystyle\int_{\mathbb{R}^d} \left \langle  T_{\mu_{0}}^{\mu_{i}} - T_{\mu_{0}}^{\eta},  T_{\mu_{0}}^{\mu_{j}} - T_{\mu_{0}}^{\eta}  \right \rangle d\mu_0.$  This reduces the problem to solving a convex quadratic program.

\begin{remark}
We note that the analysis problem for the $W_{2}$BCM can be solved in an analogous manner \textit{if the underlying measure is exactly a 2-Wasserstein barycenter of the reference measures}.  More precisely, given $\eta$ and reference measures $\{\mu_{i}\}_{i=1}^{m}$, we formulate the analysis problem for the $W_{2}$BCM \cite{werenski2022measure} as 
\begin{equation*}
    \min_{\lambda \in \Delta^m} W_{2}^{2}(\nu_{\lambda},\eta), \quad \nu_{\lambda}=\argmin_{\nu\in\mathcal{P}(\mathbb{R}^{d})}\sum_{i=1}^{m}\lambda_{i}W_{2}^{2}(\nu,\mu_{i}).
\end{equation*}If there exists a unique $\lambda^{*}\in\Delta^{m}$ such that $\eta=\nu_{\lambda}$, then \cite{werenski2022measure} proves that, under certain assumptions on $\eta$ and $\{\mu_{i}\}_{i=1}^{m}$ \cite{caffarelli1992regularity, panaretos2020invitation} the coefficients $\lambda^{*}$ satisfy the optimization problem
\begin{equation}\label{eqn:BCM_Analysis_QP}
    \lambda^{*}=\argmin_{\lambda\in\Delta^{m}} \lambda^{T}A\lambda,
\end{equation}
where $A_{ij} = \displaystyle\int_{\mathbb{R}^d} \left \langle  T_{\eta}^{\mu_{i}} - \id,  T_{\eta}^{\mu_{j}} - \id  \right \rangle d\eta.$  
\end{remark}

\subsection{Plug-in Estimation for the Synthesis and Analysis}
In order to solve (\ref{eqn:AnalysisLBCM}) or (\ref{eqn:BCM_Analysis_QP}) when measures are only accessible via samples, the key idea is to use a sample-based plug-in estimator for the necessary optimal transport maps and estimate the inner products via Monte-Carlo.  We employ the well-known entropy-regularized map \cite{pooladian2021entropic, rigollet2022sample, stromme2023sampling, masud2023multivariate, werenski2023estimation, werenski2024rank}, denoted $T_{\mu}^{\nu, \varepsilon}$, obtained via entropy-regularized optimal transport (EOT) as the plug-in estimator of the OT maps due to computational efficiency of EOT \cite{sinkhorn1967concerning, cuturi2013sinkhorn}.  We refer to Appendix \ref{sec:EOT} for necessary background on EOT and \cite{nutz2021introduction} for more details.  Its sample version $T_{\mu}^{\nu, n,\varepsilon}$ estimates $T_{\mu}^{\nu}$ with statistical rates that depend exponentially on the data dimension \cite{pooladian2021entropic}.  Estimating each optimal transport map $T_{\mu_{i}}^{\nu}, i=1,\dots,m$, allows us to estimate $A_{ij}^{L}$, which requires estimating the inner product of two OT maps. To do so, we follow the framework of \cite{pooladian2021entropic} and place Assumptions 1-3 on all measures $\{\mu_i\}_{i=0}^m$.

\begin{assumption} \label{ass : stronger density assumption}
    The measures $\mu_{i}$ have compact convex support $\Omega$ and they are absolutely continuous on this set with densities $p_i$ which satisfy $0 < m < p_i < M$ for two constants $m,M$.
\end{assumption}

\begin{assumption} \label{ass : regularity of potentials}
    For $i=1,...,m$ let $\varphi_i,\varphi_i^*$ be the optimal potentials between $\mu_0$ and $\mu_i$ as defined in Theorem \ref{thm : knott-smith and brenier}. Then $\varphi_i \in \mathcal{C}^2(\Omega)$ and $\varphi^*_i \in \mathcal{C}^{\alpha+1}(\Omega)$ for some $\alpha > 1$. 
\end{assumption}

\begin{assumption} \label{ass : regularity of hessian}
    There exist $l,L>0$ with $lI \preceq \nabla^2\varphi_i(x) \preceq LI$ for all $x \in \Omega$.
\end{assumption}

\begin{theorem}[\cite{pooladian2021entropic}, Theorem 3] \label{thm : pooladian entropic estimation} 
    Suppose that $\mu$ and $\nu$ satisfy Assumptions \ref{ass : stronger density assumption}, \ref{ass : regularity of potentials}, and \ref{ass : regularity of hessian}.  Then the plug-in estimate of the entropy-regularized map $T_{\mu}^{\nu, n,\varepsilon}$, with regularization parameter $\varepsilon \asymp n^{-1/(d + \bar{\alpha} + 1)}$, satisfies
    \begin{equation}
        \mathbb{E}\left [ \norm{T_{\mu}^{\nu, n,\varepsilon} - T_{\mu}^{\nu}}^2_{L^2(\mu)} \right ] \lesssim n^{-\frac{(\bar{\alpha} + 1)}{2(d + \bar{\alpha} + 1)}}\log n
    \end{equation}
    where $\bar{\alpha} = \min(\alpha, 3)$. The implicit constant may depend on $\mu$ and $\nu$, but does not depend on $n$.
\end{theorem}

\begin{remark}The use of Theorem \ref{thm : pooladian entropic estimation} imposes strong assumptions on $\mu,\nu$ and requires that the integrated Fisher information along the 2-Wasserstein geodesic between them be finite.  Other estimators for $T_{\mu}^{\nu}$ exist that offer improved rates of convergence under smoothness assumptions \cite{deb2021rates, manole2024plugin} and may be used as plug-in estimators.
\end{remark}

Once these maps are estimated, we can proceed to solve the synthesis and analysis problems for the LBCM as follows.

\begin{proposition}
\label{prop : synthesis}
Assume that $\mu_{0}$ and $\{\mu_{i}\}_{i=1}^{m}$ satisfy Assumptions 1, 2, and 3. For each $i$, let $T_{\mu_0}^{\mu_i,n,\varepsilon}$ be the entropic map between $\mu_0^n$ and $\mu_i^n$, where $\mu_0^n$ and $\mu_i^n$ are $n$ sample empirical measures from $\mu_{0}$ and $\mu_{i}$, respectively. Define $\nu=\left(\sum_{i=1}^{m}\lambda_{i}T_{\mu_{0}}^{\mu_{i}}\right)\#\mu_{0}$ and $\nu^n =\left(\sum_{i=1}^{m}\lambda_{i}T_{\mu_{0}}^{\mu_{i},n,\varepsilon}\right)\#\tilde{\mu}_{0}^{n}$, where $\tilde{\mu}_0^n$ is an i.i.d. sample from $\mu_0$ which is independent from $\mu_0^n.$  

Let \begin{equation*}r_{n,d}=\begin{cases}
n^{-\frac{1}{4}} & \text{if } d=1,2,3 \\
n^{-\frac{1}{4}}\sqrt{\log(n)} & \text{if }  d=4 \\
n^{-\frac{1}{d}} & \text{if } d\ge 5
\end{cases}
\end{equation*}Then 
\begin{equation*}
    \mathbb{E}W_{2}(\nu,\nu^n)\lesssim n^{-\frac{(\bar{\alpha} + 1)}{4(d + \bar{\alpha} + 1)}}\sqrt{\log n}+r_{n,d},
\end{equation*}
where $\bar{\alpha} = \min(\alpha, 3)$ and where the implicit constant may depend on $\mu_0,...,\mu_m$ but not $n$.
\end{proposition}

\begin{algorithm}[htbp!]
    \caption{Estimate $\lambda$ in the LBCM} \label{alg : estimate lambda LBCM}
    \begin{algorithmic}
        \STATE {\bfseries Input:} i.i.d. samples $X_1,...,X_{2n} \sim \mu_0, Y_1^i,...,Y_n^i \sim \mu_i: i=1,...,m, Y_1^\eta,...,Y_n^\eta \sim\eta$, regularization parameter $\varepsilon > 0$.
        \FOR{$i = 1,...,m$} 
            \STATE Solve for $g^i_\varepsilon$ as the optimal $g$ in
            \begin{equation*} \begin{split}
                &\max_{f,g} \frac{1}{n}\sum_{j=1}^n f(X_j)+ \frac{1}{n}\sum_{k=1}^n g(Y_k^i) \\
                &- \frac{\varepsilon}{n^2}\sum_{j,k}^n \expe \left(f(X_j) + g(Y_k^i) - \frac{1}{2}\norm{X_j - Y_k^i}_2^2 \right),
            \end{split}
            \end{equation*}
            and similarly for $g^\eta_\varepsilon$.
            \STATE Define $T_{\mu_{0}}^{\mu_{i},n,\varepsilon}$ through \eqref{eq : sample entropic map} with $g_\varepsilon = g_\varepsilon^i$ and $\{Y_1^i,...,Y_n^i\}$ and similarly for $T_{\mu_{0}}^{\eta, n,\varepsilon}$.
        \ENDFOR
        \STATE Set $\hat{A}^L \in \mathbb{R}^{m \times m}$ to be the matrix with entries
        \begin{align*}
            \hat{A}_{ij}^L = \frac{1}{n}\sum_{k=n+1}^{2n} \langle &T_{\mu_{0}}^{\mu_{i},n,\varepsilon}(X_k) - T_{\mu_{0}}^{\eta,n,\varepsilon}(X_k), \\&T_{\mu_{0}}^{\mu_{j},n,\varepsilon}(X_k) - T_{\mu_{0}}^{\eta,n,\varepsilon}(X_k) \rangle.
        \end{align*}
        \STATE \textbf{Return} $\hat{\lambda}=\displaystyle\argmin_{\lambda \in \Delta^m} \lambda^T\hat{A}^L\lambda$.
    \end{algorithmic}
\end{algorithm}

\begin{theorem} \label{thm : estimate of Aij lbcm}
    Suppose that Assumptions \ref{ass : stronger density assumption}, \ref{ass : regularity of potentials}, and \ref{ass : regularity of hessian} are satisfied for the pairs $(\mu_0, \mu_1), (\mu_0, \mu_2),$ and $(\mu_0,\eta)$. Let $X_1,...,X_{2n} \sim \mu_0, Y_1^1,...,Y^1_n \sim \mu_1, Y^2_1,...,Y^2_n \sim \mu_2$, and $Z_1,...,Z_n \sim \eta$ be i.i.d. samples from the respective measures. Let $T_{\mu_{0}}^{\mu_{1},n,\varepsilon},T_{\mu_{0}}^{\mu_{2},n,\varepsilon},$ and $T_{\mu_{0}}^{\eta,n,\varepsilon}$ be the plug-in estimate entropy-regularized maps from $\mu_0^n$ to $\mu_1^n, \mu_2^n,$ and $\eta^n$ respectively, all with $\varepsilon \asymp n^{-1/(d+\Bar{\alpha}+1)}$ and computed using $X_{1},...,X_{n}$. Then we have
    \begin{align*}
       &\mathbb{E}\bigg[ \bigg | \int \langle T_{\mu_{0}}^{\mu_{1}} - T_{\mu_{0}}^{\eta}, T_{\mu_{0}}^{\mu_{2}} - T_{\mu_{0}}^{\eta} \rangle d\mu_0 -\frac{1}{n} \sum_{i=n+1}^{2n} \\
       &\langle T_{\mu_{0}}^{\mu_{1},n,\varepsilon}(X_i) - T_{\mu_{0}}^{\eta, n,\varepsilon}(X_i), T_{\mu_{0}}^{\mu_{2}, n,\varepsilon}(X_i) - T_{\mu_{0}}^{\eta, n,\varepsilon}(X_i) \rangle \bigg | \bigg] \\
            &\hspace{2cm}\lesssim \frac{1}{\sqrt{n}} + n^{-\frac{\Bar{\alpha}+ 1}{4(d + \Bar{\alpha} + 1) }}\sqrt{\log n},
    \end{align*}
    where  $\Bar{\alpha} = \min(\alpha, 3)$. The implicit constant may depend on $\mu_0,\mu_1,\mu_{2},\eta$ but not $n$.
\end{theorem}

\begin{corollary} \label{cor : estimate lambda lbcm} 
    Let $\hat{\lambda}$ be the random estimate obtained from Algorithm \ref{alg : estimate lambda LBCM}. Suppose that $A^L$ has an eigenvalue of 0 with multiplicity 1 and that $\lambda_{*}\in\Delta^{m}$ realizes $\lambda_{*}^{T}A^L\lambda_{*}=0$.  Then under the assumptions of Theorem \ref{thm : estimate of Aij lbcm}, 
    \begin{equation*}
        \mathbb{E}[\|\hat{\lambda} -\lambda_{*}\|_{2}^{2}]\lesssim \frac{1}{\sqrt{n}} + n^{-\frac{\alpha + 1}{4(d + \alpha + 1)}}\sqrt{\log n}.
    \end{equation*}
    The implicit constant may depend on $\mu_0,...,\mu_m,\eta$ but not $n$.
\end{corollary}

\section{Representation Via LBCM}
\label{sec:RepresentationalCapacity}

One of the most natural questions for using the $W_{2}$BCM and the LBCM is to describe their ``representational capacity." This is intended to quantify in an appropriate sense how closely the classes of measures in $\BCM(\{\mu_i\}_{i=1}^m)$ and $\LBCM(\mu_0;\{\mu_i\}_{i=1}^m)$ can represent either an arbitrary probability measure or probability measures from a predetermined class. For the following discussion we will slightly generalize the $W_{2}$BCM and LBCM to accommodate infinite families of reference measures. For a family of measures indexed by a set $I$ we define
\begin{align*}
    \LBCM(\mu_0; \{\mu_i\}_{i \in I}) &\triangleq \{ \mu_\lambda \ | \ \lambda \in \mathcal{P}(I) \} \text{ where } \\
    \mu_\lambda &= \argmin_{\nu \in \mathcal{P}(\mathbb{R}^d)} \int \norm{T_{\mu_0}^{\nu} - T_{\mu_0}^{\mu_i}}_{L^2(\mu_0)}^2 d\lambda(i),
\end{align*}
which extends the case of finitely many measures. In an analogous way we extend the $\BCM$ by using
\begin{equation*}
    \mu_\lambda = \argmin_{\nu \in \mathcal{P}(\mathbb{R}^d)} \int W_2^2(\nu,\mu_i) d\lambda(i).
\end{equation*}

Selecting a base measure along with a set of references such that the associated $\LBCM$ or $\BCM$ well-approximates a given set of measures is a non-trivial problem. Indeed some common families which are have exceptional representational capacity for linear mixtures have extremely limited capacity when used in either the $\LBCM$ or $\BCM$. In particular, we note the following Proposition.
 
\begin{proposition} \label{prop : bcm on gaussians}
    Suppose that $\mu_{(x,\Sigma)} = \mathcal{N}(x;\Sigma)$ is a Gaussian with mean $x$ and covariance matrix $\Sigma$ and consider the index set $I = \mathbb{R}^d \times \mathbb{S}^d_{++}$. It holds that
    \begin{align*}
        \BCM(\{\mu_i\}_{i\in I}) \subset \{\mathcal{N}(x;\Sigma) \ | \ x \in \mathbb{R}^d, \Sigma \in \mathbb{S}^d_{++} \}, \\
        \LBCM(\mu_0; \{\mu_i\}_{i\in I}) \subset \{\mathcal{N}(x;\Sigma) \ | \ x \in \mathbb{R}^d, \Sigma \in \mathbb{S}^d_{++} \}.
    \end{align*}
    In contrast, the set
    \begin{equation*}
        \left \{ \int \mu_{(x,\Sigma)} d\lambda(x,\Sigma) \ | \ \lambda \in \mathcal{P}(\mathbb{R}^d \times \mathbb{S}_{++}^d) \right \}
    \end{equation*}
    is dense (in the sense of weak convergence) in $\mathcal{P}(\mathbb{R}^{d})$.
\end{proposition}

\begin{remark}
    Proposition \ref{prop : bcm on gaussians} holds not just for Gaussian reference measures, but for any scale-translation family, including Dirac masses  \cite{bonneel2015sliced}.
\end{remark}

Proposition \ref{prop : bcm on gaussians} tells us that a simple family which is dense in $\mathcal{P}(\mathbb{R}^{d})$ when combined in the linear mixture sense is far from being dense using either the $W_{2}$BCM or LBCM. This leads to the question of whether or not a family of measures $\mu_0, \{\mu_i\}_{i\in I}$ can be found where the opposite result holds true, that is a family of measures where the $W_{2}$BCM or LBCM is dense but the collection of linear mixtures is not. The result in the next section exhibits such a case in one dimension.

\subsection{Density of the LBCM in $\mathcal{P}([a,b])$}
For convenience we restrict ourselves to the interval $[0,1]$ but the construction can be extended to all of $\mathbb{R}$.
\begin{theorem} \label{thm : lbcm 1D capacity}
    Let $U([0,1])$ denote the uniform measure over the interval $[0,1]$. The sets 
    \begin{align*}
        \BCM(\{a\delta_0 + (1-a)\delta_1 \ | \ a \in [0,1] \}) \\
        \LBCM(U([0,1]); \{a\delta_0 + (1-a)\delta_1 \ | \ a \in [0,1] \})
    \end{align*}
    are both dense in $\mathcal{P}([0,1])$ with respect to weak convergence. The set  
    \begin{equation*}
        \normalfont{\text{conv}}(\{U([0,1])\} \cup \{a\delta_0 + (1-a)\delta_1 \ | \ a \in [0,1] \})
    \end{equation*}
    is not dense in $\mathcal{P}([0,1]).$
\end{theorem}
The strategy of the proof is to start by representing any measure $\mu \in \mathcal{P}([0,1])$ by its transport map $T$ from $U([0,1])$. One then constructs a map, which is almost everywhere equal to $T$, as a convex combination of the maps $T_a$, the optimal transport map from $U([0,1])$ to $a\delta_0 + (1-a)\delta_1$. The existence of such a convex combination immediately establishes that $\mu \in \LBCM(U([0,1]); \{a\delta_0 + (1-a)\delta_1 \ | \ a \in [0,1] \})$. 

From a convex analysis perspective, the proof is done by identifying the \textit{extreme} optimal transport maps in the set of transport maps from $U([0,1])$ to measures supported on $[0,1]$ as well as identify the measures that these extreme maps are associated with. In the next section we investigate if analogous results hold in higher dimensions.

\subsection{Capacity in Higher Dimensions}

In this section we will require the following definition.
\begin{definition} \label{def : sub differential}
    Let $\phi :\mathbb{R}^d \rightarrow \mathbb{R}$ be a convex function. The sub-differential $\partial \phi$ is a set-valued function
    \begin{equation*}
        \partial \phi(x) \triangleq \left \{ y \in \mathbb{R}^d \ | \ \phi(z) \geq \phi(x) + \langle y, z-x\rangle \ \forall z \in \mathbb{R}^d \right \}.
    \end{equation*}
\end{definition}

It is well-known by Rockafellar's theorem \cite{rockafellar1997convex} that convex functions are almost everywhere differentiable. The sub-differential is useful for handling those points where a convex function fails to be differentiable. 

One particular point of interest in the proof of Theorem \ref{thm : lbcm 1D capacity} is that the reference measure are themselves supported on the extreme points of the interval $[0,1]$ (namely $\{0,1\}$). A natural attempt to generalize the above example to $\mathbb{R}^d$, $d>1$, is to replace $[0,1]$ with a convex polytope $C \subset \mathbb{R}^d$, that is a compact convex set that can be realized as $C = \normalfont{\text{conv}}(\{v_i\}_{i=1}^\ell)$ where $\ell < \infty$ \footnote{We may and do assume without loss of generality that every vertex $v_i$ of $C$ is extreme.},  replace $U([0,1])$ with $U(C)$, replace $\mathcal{P}([0,1])$ with $\mathcal{P}(C)$ and replace $\{a\delta_0 + (1-a)\delta_1 \ | \ a \in [0,1]\}$ with $\mathcal{P}(\{v_i\}_{i=1}^\ell)$, the set of probability measures on the extreme points. 

The following result shows that maps with image contained in $\{v_i\}_{i=1}^\ell$ are extreme, extending the approach used in Theorem \ref{thm : lbcm 1D capacity} in one dimension.
\begin{proposition} \label{prop : maps to extreme are extreme}
    Define the sets 
    \begin{align*}
    &\mathcal{T}(C) \triangleq \left \{T\!:\!C\!\rightarrow\!C  |  \exists \varphi \normalfont{\text{ cvx s.t. }}  T(x)\!\in\! \partial \varphi(x) \ \forall x \in C \right \}, \\
        &\mathcal{V}(C) \triangleq \left \{ T \in \mathcal{T}(C) \ | \ T(x) \in \{v_i\}_{i=1}^\ell \text{ for a.e. } x \in C \right \}.
    \end{align*}
    The set $\mathcal{V}(C)$ consists of the optimal transport maps from $U(C)$ to $\mathcal{P}(\{v_i\}_{i=1}^\ell)$ and additionally $\mathcal{V}(C)$ are extreme points in $\mathcal{T}(C)$ up to almost everywhere equality in the sense that if $T \in \mathcal{V}(C)$ can be expressed as
    \begin{equation*}
        T = \int_{\mathcal{T}(C)} T_i d\lambda(T_i)
    \end{equation*}
    for some $\lambda \in \mathcal{P}(\mathcal{T}(C))$ then $\mathbb{P}(T = T_i \text{ a.e.}) = 1$ where $T_i \sim \lambda$.
\end{proposition}
The choice of $\mathcal{T}$ is to refer to the transport maps while  $\mathcal{V}$ is to refer to the vertices of the polytope $C$ which define the set $\mathcal{V}(C)$. Importantly, note that in Proposition \ref{prop : maps to extreme are extreme} it states that $\mathcal{V}(C)$ consists of points which are extreme in $\mathcal{T}(C)$, but it does \textit{not} claim that $\mathcal{V}(C)$ contains all of the extreme points. However this is the case in one dimension when $C = [0,1]$ and the set $\mathcal{V}([0,1])$ is not only a subset of the extreme points of $\mathcal{T}([0,1])$ but in fact consists of the \textit{entire} set of extreme points.  

A natural question in higher dimensions is if the set of measures on the extreme points is sufficient to generate a dense $\LBCM$. Of course the measures on the extreme points $\mathcal{P}(\{v_i\}_{i=1}^\ell)$ can be associated the set of $\mathcal{V}(C)$ through the relation $\mu[\{v_i\}] = U(C)[T^{-1}(v_i)]$ for every $T \in \mathcal{V}(C)$, or vice versa by observing that for every $\mu \in \mathcal{P}(\{v_i\}_{i=1}^\ell)$ the map $T_{U(C)}^\mu \in \mathcal{V}(C)$. This leads to the following two equivalent questions

\begin{question} \label{que : is dense}
    Is the set $\LBCM(U[C]; \mathcal{P}(\{v_i\}_{i=1}^\ell))$ dense in $\mathcal{P}(C)$?
\end{question}

\begin{question} \label{que : convex maps}
    Is it the case that $\normalfont{\text{conv}}(\mathcal{V}(C))$ is dense in $\mathcal{T}(C)$? 
\end{question}

The following result provides a negative answer to Question \ref{que : convex maps}, and thus a negative answer to Question \ref{que : is dense}. 
\begin{theorem} \label{thm : questions negative}
    The answer to Question \ref{que : convex maps} is false, that is $\normalfont{\text{conv}}(\mathcal{V}(C))$ is not dense in $\mathcal{T}(C)$. As a consequence the answer to Question \ref{que : is dense} is also false.
\end{theorem}
Theorem \ref{thm : questions negative} invites the following open question.
\begin{question} \label{que : extreme}
    What is the smallest set $\{\mu_i\}_{i \in I}$ such that $\LBCM(\text{U}(C); \{\mu_i\}_{i \in I})$ is dense in $\mathcal{P}(C)$? Equivalently, what are the extreme elements of the convex set $\mathcal{T}(C)$?
\end{question}

\section{Applications}

We consider applications of the LBCM to two problems: estimation of covariance matrices and recovery of corrupted MNIST digits\footnote{Code to recreate the experiments is publicly available at \url{https://github.com/MattWerenski/LBCM}}.  We compare against standard baselines as well as the $W_{2}$BCM approach of \cite{werenski2022measure}.  Details of our approach are in Appendices \ref{sec:Algorithm3}-\ref{sec:AdditionalExperiments}.

\subsection{Covariance Estimation}

In covariance estimation one is given a finite set of samples $X_1,...,X_n \sim \mathcal{N}(0,\Sigma)$ and is tasked with constructing an estimate $\hat{\Sigma} \approx \Sigma$ \cite{pourahmadi2013high,srivastava2013covariance,friedman2008sparse}. We consider the experimental set up in Algorithm \ref{alg : covariance set up}. 

\begin{algorithm}[htbp!]
    \caption{Covariance Experimental Set Up} \label{alg : covariance set up}
    \begin{algorithmic}[1]
        \STATE {\bfseries Input:} Number of references ($m$), random covariance procedure ($\mathtt{RandomCovariances}$), the number of samples ($n$), random coordinate procedure ($\mathtt{RandomCoordinate}$).

\vspace{5pt}
        \STATE Set $(\Sigma_i,...,\Sigma_m) = \mathtt{RandomCovariances}(m)$.
        \STATE Set $\lambda = \mathtt{RandomCoordinate}(m)$.
        \STATE Compute $\Sigma$ by using Algorithm 1 in \cite{chewi2020gradient} with parameters $\lambda,\Sigma_1,...,\Sigma_m$.
        \STATE Draw $n$ samples $X_1,...,X_n \sim \mathcal{N}(0,\Sigma)$.
        \STATE Construct estimates $\hat{\lambda}, \hat{\Sigma}$ using Algorithm \ref{alg : covariance estimate} with the samples $X_1,...,X_n$ and references $\Sigma_1,...,\Sigma_m$.
        
\vspace{5pt}

        \STATE \textbf{Return} $\hat{\lambda}, \hat{\Sigma}$.
    \end{algorithmic}
\end{algorithm}
Algorithm \ref{alg : covariance set up} requires two procedures: $\mathtt{RandomCovariance}$, which returns $m$ positive definite matrices representing the covariances of the reference measures, and $\mathtt{RandomCoordinate}$ which returns a random vector in $\Delta^m$. Throughout we only consider sampling from $\Delta^m$ uniformly. The estimates of $\lambda$ and $\Sigma$ are constructed using a covariance estimation procedure which is specified in Algorithm \ref{alg : covariance estimate}.

\begin{algorithm}[htbp!]
    \caption{Covariance Estimation with References} \label{alg : covariance estimate}
    \begin{algorithmic}[1]
        \STATE {\bfseries Input:} i.i.d. samples ($X_1,...,X_{n} \sim \mathcal{N}(0,\Sigma)$), reference measures ($\mathcal{N}(0,\Sigma_1),...,\mathcal{N}(0,\Sigma_m)$), coordinate estimation procedure ($\mathtt{EstimateCoordinate}$), covariance estimation procedure ($\mathtt{EstimateCovariance}$).

        \vspace{5pt}

        \STATE Set $\hat{\Sigma}_{\text{emp}} = \frac{1}{n} \sum_{i=1}^n X_iX_i^T$.
        \STATE Set $\hat{\lambda} = \mathtt{EstimateCoordinate}(\hat{\Sigma}_{\text{emp}};\Sigma_1,...,\Sigma_m)$.
        \STATE Set $\hat{\Sigma} = \mathtt{EstimateCovariance}(\hat{\lambda};\Sigma_1,...,\Sigma_m)$.
        
        \vspace{5pt}

        \STATE \textbf{Return} $\hat{\lambda}, \hat{\Sigma}$.
    \end{algorithmic}
\end{algorithm}
Algorithm \ref{alg : covariance estimate} requires two further procedures to be provided. The first is $\mathtt{EstimateCoordinate}$ which takes as arguments the empirical covariance matrix $\hat{\Sigma}_{\text{emp}}$ and the collection of reference covariances $\Sigma_{1},...,\Sigma_m$ and returns an estimate of the coordinate $\hat{\lambda}$. The second procedure is $\mathtt{EstimateCovariance}$ which takes as arguments the estimated coordinate and the set of reference covariances and returns the estimate of the covariance. 

We consider three versions of Algorithm \ref{alg : covariance estimate}: LBCM, BCM, and maximum likelihood estimation (MLE); details are in Appendix \ref{sec:Algorithm3}.

\begin{remark} \label{rem : simultaneously diagonalizable}
    In order to effectively compare the three methods above we must consider a setting where the $W_{2}$BCM and LBCM coincide. This is the case when the covariance matrices $\Sigma_1,...,\Sigma_m$ are \textit{simultaneously diagonalizable} which is equivalent to $\Sigma_1,...,\Sigma_m$ having the same eigenvectors. The $\mathtt{RandomCovariances}$ achieves this using a two step process. First, sample an orthogonal matrix $O \sim \text{Ortho}(d)$ and then for each $i = 1,...,m$ set $\Sigma_i = O^TD_iO$ where $D_i$ is a random diagonal matrix with strictly positive entries on its diagonal.  Throughout we let $D_i$ have independent diagonal entries distributed according to the half-normal distribution.\footnote{If $X \sim \mathcal{N}(0,1)$ then $|X|$ has a half-normal distribution.}
\end{remark}

Results are shown in Figure \ref{fig:GaussianResults}.  We plot both the discrepancy between the recovered covariance matrix and the discrepancy in the estimate of the coordinate $\lambda$. In addition we plot the distance from $\mathcal{N}(0,\hat{\Sigma}_{\text{emp}})$ to $\mathcal{N}(0,\Sigma)$. The LBCM uses $\mu_0 = \mathcal{N}(0,I)$.  The first setting we consider is when $m = 10$ and $d = 10$.  In this case, incorporating the $W_{2}$BCM or LBCM leads to approximately one quarter the error obtained by the empirical covariance.  The asymptotic behavior for the $W_{2}$BCM, LBCM, and empirical covariance are all essentially of order $n^{-1/2}$. The quality of the estimate of the covariance is extremely similar for both the $W_{2}$BCM and LBCM in this setting. The behavior of the MLE is interesting but we do not have a precise explanation of its behavior which may be either theoretical or connected with the numerical methods involved in obtaining it.  Next we follow the same procedure except with $d=20$ and observe similar results. 

\begin{figure}[t!]
    \centering
    \includegraphics[width=.9\linewidth]{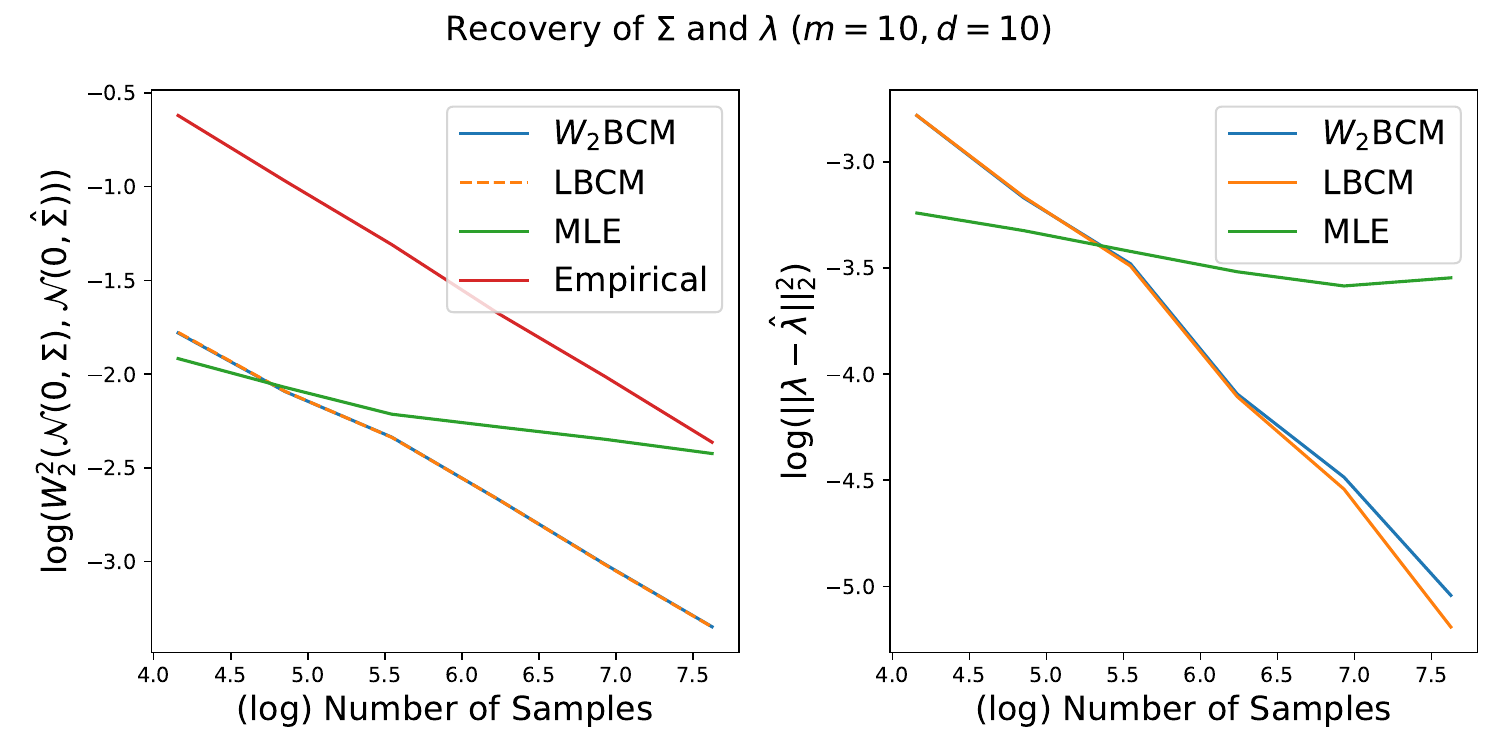}
    \includegraphics[width=.9\linewidth]{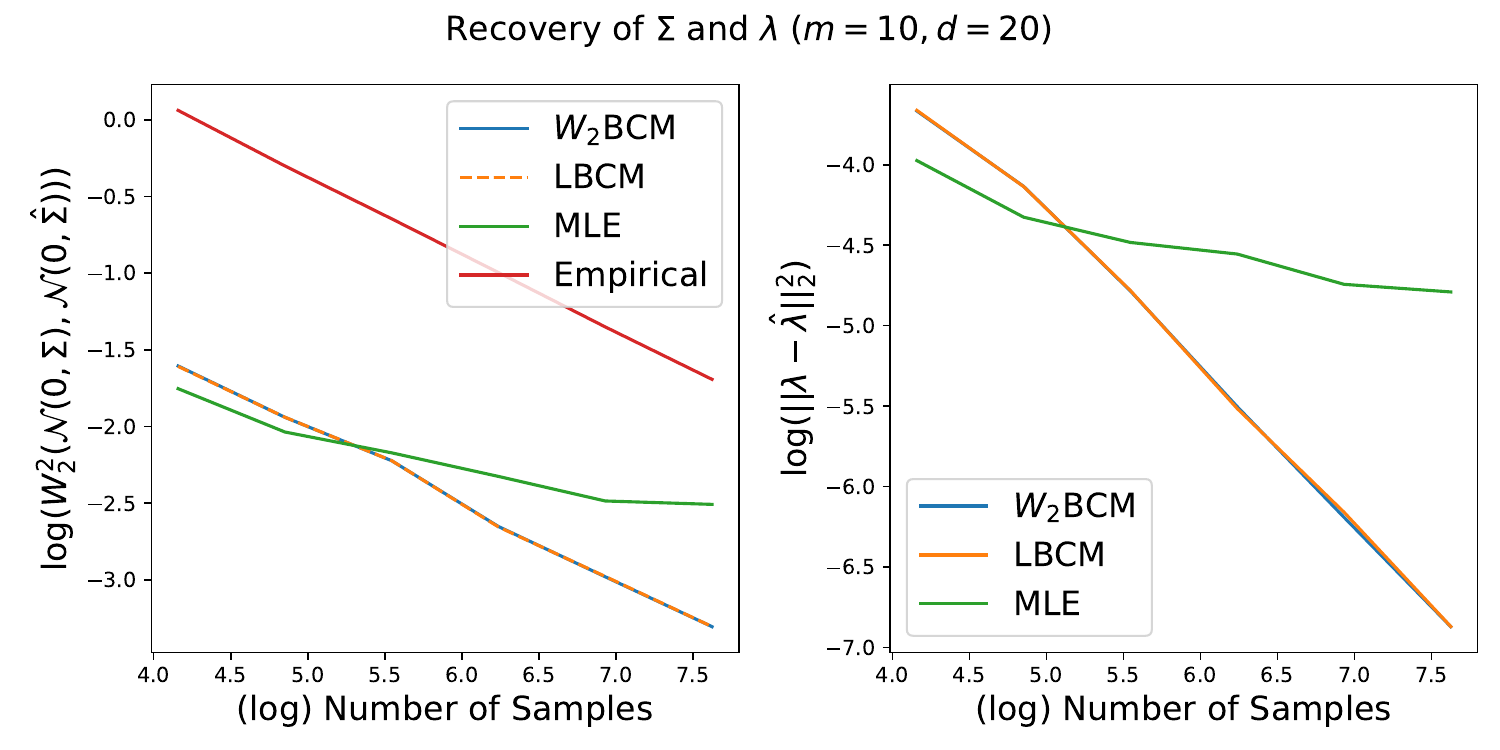}
    \caption{\label{fig:GaussianResults} Recovery of the covariance matrix and coordinate in logarithmic scale.}
\end{figure}
\subsection{MNIST Reconstruction}

Another application we consider is in digit reconstruction. This is done using the MNIST image set which consists of 28$\times$28 grayscale images of handwritten digits 0 through 9. Sample digits are given as the first column of Figures \ref{fig : mnist rebuilds} and \ref{fig : lbcm reconstructions}.  The problem we consider is, given an occluded digit which is missing the central block of pixels, to recover the digit (see Figure \ref{fig : mnist rebuilds}). We do so using the $W_{2}$BCM, LBCM, and a simple linear method. Each of these methods falls into the framework of Algorithm \ref{alg : mnist rebuild} by specifying $\mathtt{EstimateCoordinate}$ and $\mathtt{SynthesizeImage}$.

\begin{algorithm}[htbp!]
    \caption{Digit Reconstruction Procedure} \label{alg : mnist rebuild}
    \begin{algorithmic}[1]
        \STATE {\bfseries Input:} Corrupted digit to reconstruct ($\mathtt{C}_\eta$), Reference digits ($\mathtt{D}_1,...,\mathtt{D}_m$), Coordinate estimation procedure ($\mathtt{EstimateCoordinate}$), Image synthesis procedure ($\mathtt{SynthesizeImage}$).
        \STATE Set $\mathtt{C}_i = \mathtt{Occlude}(\mathtt{D}_i)$ for $i=1,...,m$ $\hspace{0,5cm}\#$ Removes the central block of pixels
        \STATE Set $\lambda = \mathtt{EstimateCoordinate}(\mathtt{C}_\eta, \mathtt{C}_1,...,\mathtt{C}_m)$
        \STATE \textbf{Return} $\mathtt{SynthesizeImage}(\lambda, [\mathtt{D}_1,...,\mathtt{D}_m])$
    \end{algorithmic}
\end{algorithm}

For the LBCM we specify a base digit $\mathtt{D}_0$ and an occluded version $\mathtt{C}_0$. $\mathtt{EstimateCoordinate}$ is then performed by converting $\mathtt{C}_\eta, \mathtt{C}_0,...,\mathtt{C}_m$ to measures using Algorithm \ref{alg : image to measure} and then obtain $\lambda$ by solving the minimization in \eqref{eqn:AnalysisLBCM}. For $\mathtt{SynthesizeImage}$ we convert $\mathtt{D}_0,...,\mathtt{D}_m$ to measures, again using Algorithm \ref{alg : image to measure}, and construct a measure $\rho_\lambda$ by pushing $\mathtt{D}_0$ through the weighted combination of the transport maps. We then convert the resulting measure back to an image using Algorithm \ref{alg : measure to image} in the Appendix which is akin to kernel density estimation.

The linear method we consider performs $\mathtt{EstimateCoordinate}$ by normalizing the occluded digits $\mathtt{C}_\eta,\mathtt{C}_1,...,\mathtt{C}_m$ to  each  have total mass one and then treats them as vectors in $\mathbb{R}^{784}$ with each entry corresponding to the color at a pixel position and projects $\mathtt{C}_\eta$ onto the convex hull $\normalfont{\text{conv}}(\mathtt{C}_1,...,\mathtt{C}_m)$ and returns $\lambda$ corresponding to the closest point. $\mathtt{SynthesizeImage}$ returns the weighted sum of the digits $\sum_{i=1}^m \lambda_i \mathtt{D}_i$.

Finally, for the $W_2$BCM, $\mathtt{EstimateCoordinate}$ converts $\mathtt{C}_\eta,\mathtt{C}_1,...,\mathtt{C}_m$ into measures using Algorithm \ref{alg : image to measure} in the Appendix and then obtains $\lambda$ by solving the minimization in \eqref{eqn:BCM_Analysis_QP}. To recover an image, we first synthesize a (approximate) 2-Wasserstein barycenter using Algorithm \ref{alg:WGD} (see Subsection \ref{subsect:WGD details} for details), with $\alpha=0.05$, $k=200$ and initialized at $\mu_0\triangleq\sum_{i=1}^m\lambda_i\mathtt{D}_i$, the weighted sum of digits computed via the linear method. We recover an image by applying $\mathtt{SynthesizeImage}$ to the output.

Examples of running this procedure with 10 given reference 4's are shown in Figure \ref{fig : mnist rebuilds}. The first two columns correspond to the original digit before and after being occluded. The second and third columns are the reconstruction using the LBCM and two different base digits (see Appendix \ref{sec:AdditionalExperiments}). The fifth column uses the $W_2$BCM and the sixth column is the linear reconstruction. Above each column, we report the (average) run time to produce each image, as well as the $W_2^2$ cost between the reconstructed image and the original image as probability measures.  We note that both the $W_2$BCM and LBCM tend to do a reasonable job recovering the digit while the linear method has significant difficulties. This indicates that the non-linear methods employed in the $W_2$BCM and LBCM may be better suited to this task than the naive linear method. We also note that the choice of initialization for the LBCM has an impact on the quality of the reconstruction, which we explore further in Appendix \ref{sec:AdditionalExperiments}. Finally, we remark that while the $W_2$BCM reconstruction has slightly lower average loss than the LBCM, the average run time is significantly greater than for the LBCM method.

\begin{figure}[t!]
    \centering
    \includegraphics[angle=-90, width=0.9\linewidth]{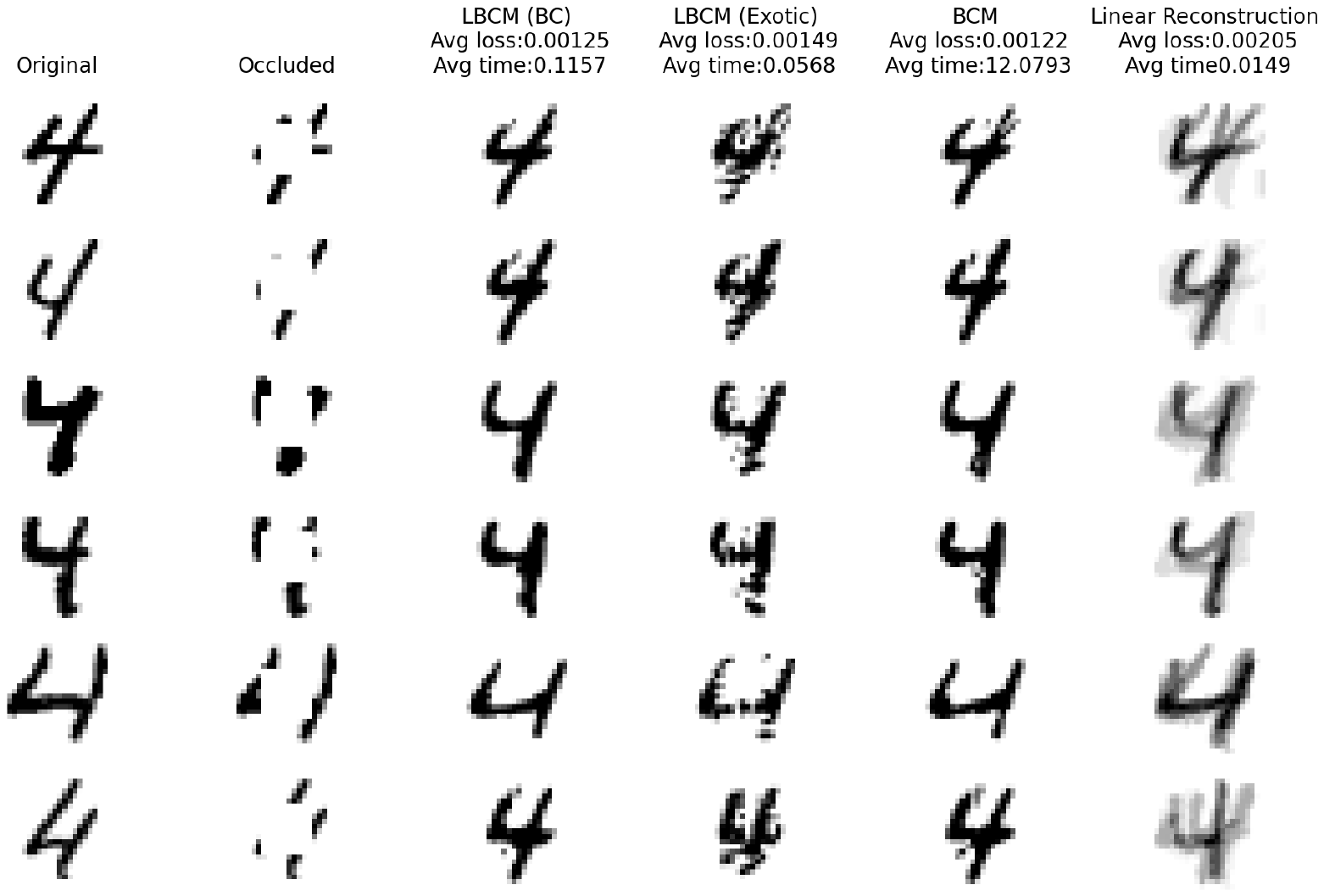}
    \caption{Reconstruction of occluded digits using the $W_{2}$BCM, LBCM, and a linear method. The base measure in Column 4 is the Double Checker in Figure \ref{fig : base measures}.}
    \label{fig : mnist rebuilds}
\end{figure}

\section{Conclusion}

While Theorem \ref{thm : questions negative} establishes that it is not possible to represent an arbitrary optimal transport map as a convex combination of transport maps onto the extreme points, is is interesting to consider whether maps that push onto richer sets (e.g., the boundary) are sufficient.

In \cite{jiang2023algorithms}, quantitative rates of approximation are given for representing 1-dimensional optimal transport maps as linear combinations of polynomials.  Developing similar results in our setting and in higher dimensions is of interest.  This is related to the problem of developing function classes well-suited to approximation of transport maps.

\noindent \textbf{Acknowledgments:}  Matthew Werenski was supported by NSF CCF 1553075, NSF DMS 2309519, and the Camille \& Henry Dreyfus Foundation. Brendan Mallery was supported by NSF CCF 1553075, NSF DMS 2309519, and NSF DMS 2318894. Shuchin Aeron would like to acknowledge partial support by NSF CCF 1553075, NSF DMS 2309519, and from NSF PHY 2019786 (The NSF AI Institute for Artificial Intelligence and Fundamental Interactions, http://iaifi.org/). James M. Murphy would like to acknowledge partial support from NSF DMS 2309519, NSF DMS 2318894, and support from the Camille \& Henry Dreyfus Foundation.

\bibliography{ref}{}
\bibliographystyle{abbrv}
\appendix

\section{Proofs from Section \ref{sec:BasicProperties}}
\label{sec:ProofsSection3}

\subsection{Proof of Proposition \ref{prop:LBCM_VariationalForm}}
For any $x$ such that $T_{\mu_{0}}^{\mu_{i}}(x)$ exists for each $i=1,\dots,m$, the pointwise variational problem 

\begin{equation*}
\argmin_{y\in\mathbb{R}^{d}}\sum_{i=1}^{m}\lambda_{i}\|y-T_{\mu_{0}}^{\mu_{i}}(x)\|^{2}
\end{equation*}has unique solution $ \sum_{i=1}^m \lambda_i T_{\mu_{0}}^{\mu_{i}}(x)$.  In particular, since the transport maps exist $\mu_{0}-$almost everywhere and are necessarily in $L^{2}(\mu_{0})$,

\begin{equation*}
\sum_{i=1}^m \lambda_i T_{\mu_{0}}^{\mu_{i}}=\argmin_{f\in L^{2}(\mu_{0})}\sum_{i=1}^{m}\lambda_{i}\|f-T_{\mu_{0}}^{\mu_{i}}\|_{L^{2}(\mu_{0})}^{2}.
\end{equation*}We now argue that in fact $\sum_{i=1}^m \lambda_i T_{\mu_{0}}^{\mu_{i}}$ is equal to $T_{\mu_{0}}^{\left( \sum_{i=1}^m \lambda_i T_{\mu_{0}}^{\mu_{i}} \right )\#\mu_{0}}$.  Notice that $\sum_{i=1}^m \lambda_i T_{\mu_{0}}^{\mu_{i}}$ is the convex combination of 2-Wasserstein optimal transport maps, each of which is by Brenier's theorem the gradient of a convex function.  It follows that $\sum_{i=1}^m \lambda_i T_{\mu_{0}}^{\mu_{i}}$ is itself the gradient of a convex function.  Moreover, $\sum_{i=1}^m \lambda_i T_{\mu_{0}}^{\mu_{i}}$ pushes $\mu_{0}$ onto $\left ( \sum_{i=1}^m \lambda_i T_{\mu_{0}}^{\mu_{i}} \right )\#\mu_0$.  Again by Brenier's theorem, it follows that 

\begin{equation*} \sum_{i=1}^m \lambda_i T_{\mu_{0}}^{\mu_{i}} = T_{\mu_{0}}^{\left ( \sum_{i=1}^m \lambda_i T_{\mu_{0}}^{\mu_{i}} \right )\#\mu_0}.
\end{equation*}
It follows that $\left(\sum_{i=1}^m \lambda_i T_{\mu_{0}}^{\mu_{i}} \right )\#\mu_0$ solves (\ref{eqn:LBCM-Barycentric-Formulation}).

\subsection{Proof of Proposition \ref{thm:LBCM_Equiv_BCM}}
\label{subsec:LBCM_Equiv_BCM}

This result relies on the equivalence of the 2-Wasserstein barycenter and the solution to the \textit{multimarginal optimal transport (MMOT)} \cite{agueh2011barycenters}.  Let $\{\mu_{i}\}_{i=1}^{m}$ be reference measures and $\lambda\in\Delta^{m}$.  The MMOT problem solves 

\begin{equation}
\label{eqn:MMOT}
\argmin_{\pi\in\Pi(\mu_{1},\dots,\mu_{m})}\int_{\mathbb{R}^{d}\times\dots\times\mathbb{R}^{d}}\sum_{i<j}\lambda_{i}\lambda_{j}\|x_{i}-x_{j}\|^{2}d\pi(x_{1},\dots,x_{m})
\end{equation}where $\Pi(\mu_{1},\dots,\mu_{m})$ is the space of measures in $\mathbb{R}^{d}\times\dots\times\mathbb{R}^{d}$ with $i^{th}$ $d$-dimensional marginal equal to $\mu_{i}$.  The MMOT problem can be related to 2-Wasserstein barycenters via the averaging operator $M_{\lambda}:\mathbb{R}^{d}\times\dots\times\mathbb{R}^{d}\rightarrow\mathbb{R}^{d}$ given by $M_{\lambda}(x_{1},\dots,x_{m})\triangleq\sum_{i=1}^{m}\lambda_{i}x_{i}$.  Indeed, 

\begin{align*}
    &\int_{\mathbb{R}^{d}\times\dots\times\mathbb{R}^{d}}\sum_{i<j}\lambda_{i}\lambda_{j}\|x_{i}-x_{j}\|^{2}d\pi(x_{1},\dots,x_{m})\\
    =&\int_{\mathbb{R}^{d}\times\dots\times\mathbb{R}^{d}}\sum_{i=1}^{m}\lambda_{i}\|x_{i}-M_{\lambda}(x_{1},\dots,x_{m})\|^{2}d\pi(x_{1},\dots,x_{m}). 
\end{align*}
\begin{proposition}(Proposition 4.2 in \cite{agueh2011barycenters} and Proposition 3.1.2 in \cite{panaretos2020invitation})
    Let $\{\mu_{i}\}_{i=1}^{m}\subset\mathcal{P}(\mathbb{R}^{d})$ and $\lambda\in\Delta^{m}$.  Then $\mu_{*}$ is a 2-Wasserstein barycenter for $\{\mu_{i}\}_{i=1}^{m}$ with coefficieints $\lambda$ if and only if there exists a solution $\pi_{*}$ to (\ref{eqn:MMOT}) such that $\mu_{*}=M_{\lambda}\#\pi_{*}$. 
\end{proposition}
\begin{proof}
 Let $\pi\in\Pi(\mu_{1},\dots,\mu_{m})$ be arbitrary and let $\mu=M_{\lambda}\#\pi$.  Then the map $S_{i}(x_{1},\dots,x_{m})=(x_{i},M(x_{1},\dots,x_{m}))$ generates a coupling between $\mu_{i}$ and $\mu$ via $S_{i}\#\pi$.  Thus
 \begin{align*}
     \int_{(\mathbb{R}^{d})^{m}}\|x_{i}-M_{\lambda}(x_{1},\dots,x_{m})\|^{2}d\pi(x_{1},\dots,x_{m})\ge W_{2}^{2}(\mu,\mu_{i}).
 \end{align*}
Summing over $\lambda$, we see
\begin{align}
\min_{\pi\in\Pi(\mu_{1},\dots,\mu_{m})}\int_{(\mathbb{R}^{d})^{m}}\sum_{i=1}^{m}\lambda_{i}\|x_{i}-M_{\lambda}(x_{1},\dots,x_{m})\|^{2}d\pi(x_{1},\dots,x_{m})\ge \min_{\mu\in\mathcal{P}(\mathbb{R}^{d})}\sum_{i=1}^{m}\lambda_{i}W_{2}^{2}(\mu,\mu_{i}).\label{eqn:W2<=MMOT}
\end{align}
 On the other hand, let $\mu\in\mathcal{P}(\mathbb{R}^{d})$ be arbitrary and for each $i=1,\dots,m$ let $\pi^{i}$ be the optimal 2-Wasserstein coupling between $\mu$ and $\mu_{i}$. We may glue the $\pi^{i}$ via their common marginal to get a measure $\eta$ on $(\mathbb{R}^{d})^{m+1}$ with marginals $\mu_{1},\dots,\mu_{m},\mu$ and its projection $\pi$ onto the first $m$ $d$-dimensional marginals is an element of $\Pi(\mu_{1},\dots,\mu_{m})$.  We now note 
\begin{align*}
    &\sum_{i=1}^{m}\lambda_{i}W_{2}^{2}(\mu,\mu_{i})\\
    =&\sum_{i=1}^{m}\lambda_{i}\int_{\mathbb{R}^{d}\times\mathbb{R}^{d}}\|x_{i}-y\|^{2}d\pi^{i}(y,x_{i})\\=&\int_{(\mathbb{R}^{d})^{m+1}}\sum_{i=1}^{m}\lambda_{i}\|x_{i}-y\|^{2}d\eta(x_{1},\dots,x_{m},y)\\
    \ge & \int_{(\mathbb{R}^{d})^{m+1}}\sum_{i=1}^{m}\lambda_{i}\|x_{i}-M_{\lambda}(x_{1},\dots,x_{m})\|^{2}d\eta(x_{1},\dots,x_{m},y)\\
    =&\int_{(\mathbb{R}^{d})^m}\sum_{i=1}^{m}\lambda_{i}\|x_{i}-M_{\lambda}(x_{1},\dots,x_{m})\|^{2}d\pi(x_{1},\dots,x_{m}),\\
\end{align*}
with equality if and only if $y=M_{\lambda}(x)$ $\eta$-a.e., so that $\mu=M_{\lambda}\#\pi$.  This establishes 

\begin{equation}
\label{eqn:W2=MMOT}
\min_{\pi\in\Pi(\mu_{1},\dots,\mu_{m})}\int_{(\mathbb{R}^{d})^{m}}\sum_{i=1}^{m}\lambda_{i}\|x_{i}-M_{\lambda}(x_{1},\dots,x_{m})\|^{2}d\pi(x_{1},\dots,x_{m})=\min_{\mu\in\mathcal{P}(\mathbb{R}^{d})}\sum_{i=1}^{m}\lambda_{i}W_{2}^{2}(\mu,\mu_{i}).
\end{equation}

Now, if a barycenter $\mu_{*}$ does not equal $M_{\lambda}\#\pi$ with $\pi$ as above, then 
\begin{align*}
    &\sum_{i=1}^{m}\lambda_{i}W_{2}^{2}(\mu_{*},\mu_{i})\\
    >&\int_{(\mathbb{R}^{d})^{m}}\sum_{i=1}^{m}\lambda_{i}\|x_{i}-M_{\lambda}(x_{1},\dots,x_{m})\|^{2}d\pi(x_{1},\dots,x_{m})\\
    \ge &\sum_{i=1}^{m}\lambda_{i}W_{2}^{2}(M_{\lambda}\#\pi,\mu_{i}),
\end{align*}which contradicts optimality of $\mu_{*}$. 
 Now, let $\pi_{*}$ denote this choice such that $M_{\lambda}\#\pi_{*}=\mu_{*}$.  Then by (\ref{eqn:W2=MMOT}) $\pi_{*}$ must be optimal for the MMOT problem:

 \begin{align*}
    &\int_{(\mathbb{R}^{d})^{m}}\sum_{i=1}^{m}\lambda_{i}\|x_{i}-M_{\lambda}(x_{1},\dots,x_{m})\|^{2}d\pi_{*}(x_{1},\dots,x_{m})\\
     =&\sum_{i=1}^{m}\lambda_{i}W_{2}^{2}(\mu_{*},\mu_{i})\\
     =&\min_{\mu\in\mathcal{P}(\mathbb{R}^{d})}\sum_{i=1}^{m}\lambda_{i}W_{2}^{2}(\mu,\mu_{i})\\
=&\min_{\pi\in\Pi(\mu_{1},\dots,\mu_{m})}\int_{(\mathbb{R}^{d})^{m}}\sum_{i=1}^{m}\lambda_{i}\|x_{i}-M_{\lambda}(x_{1},\dots,x_{m})\|^{2}d\pi(x_{1},\dots,x_{m})
 \end{align*}
 
 Conversely, if $\pi_{*}$ is an optimal coupling for the MMOT problem, then 
\begin{align*}
    &\sum_{i=1}^{m}\lambda_{i}W_{2}^{2}(M_{\lambda}\#\pi_{*},\mu_{i})\\
    \le &\sum_{i=1}^{m}\lambda_{i}
    \int_{(\mathbb{R}^{d})^{m}}\sum_{i=1}^{m}\|x_{i}-M_{\lambda}(x_{1},\dots,x_{m})\|^{2}d\pi_{*}(x_{1},\dots,x_{m})\\
    =&\min_{\pi\in\Pi(\mu_{1},\dots,\pi_{m})}\int_{(\mathbb{R}^{d})^{m}}\sum_{i=1}^{m}\|x_{i}-M_{\lambda}(x_{1},\dots,x_{m})\|^{2}d\pi(x_{1},\dots,x_{m})\\
    =&\min_{\mu\in\mathcal{P}(\mathbb{R}^{d})}\sum_{i=1}^{m}\lambda_{i}W_{2}^{2}(\mu,\mu_{i}),
\end{align*}so that $M_{\lambda}\#\pi_{*}$ is necessarily a 2-Wasserstein barycenter.

\end{proof}

With this result, we can now establish the equivalence of the $W_{2}$BCM and LBCM for compatible measures.\\

\begin{proof}[\textbf{Proof of Proposition \ref{thm:LBCM_Equiv_BCM}}]  It is enough to show that if $\{\mu_{i}\}_{i=0}^{m}$ are compatible, then the 2-Wasserstein barycenter of $\{\mu_{i}\}_{i=1}^{m}$ with coefficients $\lambda\in\Delta^{m}$ is $\left(\sum_{i=1}^{m}\lambda_{i}T_{\mu_{0}}^{\mu_{i}}\right)\#\mu_{0}$.  Note that $\pi_{*}=(T_{\mu_{0}}^{\mu_{1}},T_{\mu_{0}}^{\mu_{2}},\dots,T_{\mu_{0}}^{\mu_{m}})\#\mu_{0}\in\Pi(\mu_{1},\mu_{2},\dots,\mu_{m})$ is such that 

\begin{align*}
    &\int_{(\mathbb{R}^{d})^{m}}\|x_{i}-x_{j}\|^{2}d\pi_{*}(x_{1},\dots,x_{m})\\
    =&\int_{\mathbb{R}^{d}}\|T_{\mu_{0}}^{\mu_{i}}(x_{0})-T_{\mu_{0}}^{\mu_{j}}(x_{0})\|^{2}d\mu_{0}(x_{0})\\
    =&\int_{\mathbb{R}^{d}}\|T_{\mu_{0}}^{\mu_{i}}\circ T_{\mu_{j}}^{\mu_{0}}(x_{j})-T_{\mu_{0}}^{\mu_{j}}\circ T_{\mu_{j}}^{\mu_{0}}(x_{j})\|^{2}d\mu_{j}(x_{j})\\
    =&\int_{\mathbb{R}^{d}}\|T_{\mu_{j}}^{\mu_{i}}(x_{j})-x_{j}\|^{2}d\mu_{j}(x_{j})\\
    =&W_{2}^{2}(\mu_{i},\mu_{j}).
\end{align*}

Noting that 

\begin{align*}
\min_{\pi\in\Pi(\mu_{1},\dots,\mu_{m})}\int_{(\mathbb{R}^{d})^{m}}\sum_{i<j}\lambda_{i}\lambda_{j}\|x_{i}-x_{j}\|^{2}d\pi(x_{1},\dots,x_{m})\ge \sum_{i<j}\lambda_{i}\lambda_{j}W_{2}^{2}(\mu_{i},\mu_{j}),
\end{align*}
we see that $\pi_{*}$ is optimal for the MMOT problem, and thus $M_{\lambda}\#\pi_{*}$ is the 2-Wasserstein barycenter of $\{\mu_{i}\}_{i=1}^{m}$.  The result follows by noting $M_{\lambda}\#\pi_{*}=\left(\sum_{i=1}^{m}\lambda_{i} T_{\mu_{0}}^{\mu_{i}}\right)\#\mu_{0}$. \end{proof}

\section{Estimation of OT Maps Via Entropy-Regularized OT}
\label{sec:EOT}

The \textit{entropy-regularized optimal transport (EOT)} problem is:
\begin{equation}\label{eq : entropic problem}
    S_\varepsilon(\mu,\nu) \triangleq \min_{\pi \in \Pi(\mu,\nu)} \int \frac{1}{2}\norm{x-y}_2^2 d\pi(x,y) + \varepsilon  \int \log \left(\frac{d\pi}{d\mu\otimes d\nu}\right)(x,y) d\pi(x,y)
\end{equation}
where $\varepsilon > 0$.

\begin{theorem}(\cite{csiszar1975divergence, mena2019statistical})\label{thm : eot master theorem}
    Assume that $\mu$ and $\nu$ have finite second moments. Then the following hold. 
    \begin{enumerate}
        \item There exists a measure $\pi_\varepsilon$ which achieves the minimum in \eqref{eq : entropic problem}.
        \item The dual problem to \eqref{eq : entropic problem} is given by
        \begin{align} \label{eq : dual entropic problem}
            S_\varepsilon(\mu,\nu) &= \max_{f,g \in L^1(\mu) \times L^1(\nu)}  \Psi^{\mu,\nu}_\varepsilon(f,g), \\
            \Psi^{\mu,\nu}_\varepsilon(f,g) &\triangleq \int f(x) d\mu(x) + \int g(y) d\nu(y) \notag \\
            &\hspace{0.5cm} - \varepsilon \iint \expe\left ( f(x) + g(y) - \frac{1}{2}\norm{x-y}_2^2 \right ) d\mu(x)d\nu(y) + \varepsilon. \label{eq : dual objective}
        \end{align}
        \item There exists an optimal pair $f_\varepsilon,g_\varepsilon$ which are unique up to a.e. equality and translation by an additive constant, that is ($f_\varepsilon + a, g_\varepsilon - a)$ is also optimal for any $a \in \mathbb{R}$. 
        \item The primal-dual optimality relation is given by the following holding for $\mu\otimes\nu$ almost every ($x,y)$:
        \begin{equation} \label{eq : primal dual}
            \frac{d\pi_\varepsilon}{d\mu \otimes d\nu}(x,y) = \expe(f_\varepsilon(x) + g_\varepsilon(y) - \frac{1}{2}\norm{x-y}_2^2). 
        \end{equation}
        \item The optimal pair $f_\varepsilon,g_\varepsilon$ can be taken such that
        \begin{align*}
            f_\varepsilon(x) &= -\varepsilon \log \left ( \int \expe(g_\varepsilon(y) - \frac{1}{2}\norm{x-y}_2^2) d\nu(y) \right ) \ \forall x \in \mathbb{R}^d, \label{eq : f relation} \\
            g_\varepsilon(y) &= -\varepsilon \log \left ( \int \expe(f_\varepsilon(x) - \frac{1}{2}\norm{x-y}_2^2) d\mu(x) \right ) \ \forall y \in \mathbb{R}^d. 
        \end{align*}
    \end{enumerate}
\end{theorem}

Point 1-4 are well-known and can be found in \cite{csiszar1975divergence}.  Point 5 can be found in \cite{mena2019statistical}. We will refer to $f_\varepsilon,g_\varepsilon$, the maximizers in the dual problem \eqref{eq : dual entropic problem}, as the optimal potentials. Importantly in point 5, the equality is for all $x$ (or $y$) instead of just almost every $x$ (or $y$). We will only consider potentials which satisfy point 5.

Unlike in the unregularized OT problem in which an optimal transport map realizing (\ref{eqn:W2}) exists when $\mu$ is absolutely continuous, there is no transport map on whose graph $\pi_{\varepsilon}$ concentrates.  However, we can construct one from $f_{\varepsilon}, g_{\varepsilon}$:

\begin{definition} \label{def : entropic map}
    The \textit{entropy-regularized map} $T_{\mu}^{\nu,\varepsilon}$ for $\mu$ and $\nu$ is defined as
    \begin{equation} \label{eq : entropic map}
        T_{\mu}^{\nu, \varepsilon}(x) \triangleq \int y \expe(f_\varepsilon(x) +  g_\varepsilon(y) - \frac{1}{2}\norm{x-y}_2^2) d\nu(y).
    \end{equation}

\end{definition}

When $\mu^n$ and $\nu^n$ are empirical measures of $\mu$ and $\nu$ supported on points $X_{1},\dots,X_{n}$ and $Y_{1},\dots,Y_{n}$ respectively, we get the plug-in estimate of $T_{\mu}^{\nu, \varepsilon}$ 
        \begin{equation} \label{eq : sample entropic map}
            T_{\mu}^{\nu,n,\varepsilon}(x) \triangleq \sum_{j=1}^n \expe(f_{n,\varepsilon}(x) + g_{n,\varepsilon}(Y_j) - \frac{1}{2}\norm{x - Y_j}_2^2)Y_j,
        \end{equation}where $f_{n, \varepsilon},g_{n, \varepsilon}$ are the (extended) entropic dual potentials for the measures $\mu^{n},\nu^{n}$.

\section{Proofs from Section \ref{sec:SynthesisAnalysis}}

\subsection{Proof of Proposition \ref{prop : synthesis}}
We will require the following bound for estimating $\mathbb{E}[W_2(\mu,\mu^n)]$ on a compact set:
\begin{theorem}(Theorem 1, \cite{fournier2015rate})\label{thm:wasserstein law of large numbers}
Let $\mu\in \mathcal{P}(\Omega)$, where $\Omega\subset \mathbb{R}^d$ is a compact set, and let $\mu^n$ denote the empirical measure on $n$ samples from $\mu.$ Then:
\begin{equation*}
\mathbb{E}[W_2(\mu,\mu^n)]\leq C_{\Omega,d} \begin{cases}
n^{-\frac{1}{4}} & \text{if } d=1,2,3 \\
n^{-\frac{1}{4}}\sqrt{\log(n)} & \text{if }  d=4 \\
n^{-\frac{1}{d}} & \text{if } d\ge 5
\end{cases}
\end{equation*}
where $C_{\Omega,d}$ is a constant depending on the dimension and diameter of $\Omega$.
\end{theorem}

\begin{proof}[Proof of Proposition \ref{prop : synthesis}]
Recall that for any measures $\mu_{0}, \mu,\nu$, $W_{2}(\mu,\nu)\le \|T_{\mu_{0}}^{\mu}-T_{\mu_{0}}^{\nu}\|_{L^{2}(\mu_{0})}$.  We now calculate and apply Theorem \ref{thm : pooladian entropic estimation}:
\begin{align}
\mathbb{E}[W_{2}(\nu,\nu^n)]=&\mathbb{E}\left(W_{2}\left(\left(\sum_{i=1}^{m}\lambda_{i}T_{\mu_{0}}^{\mu_{i}}\right)\#\mu_{0}, \left(\sum_{i=1}^{m}\lambda_{i}T_{\mu_{0}}^{\mu_{i},n,\varepsilon}\right)\#\tilde{\mu}_{0}^{n}\right)\right)\notag \\
\le & \mathbb{E}\left(W_{2}\left(\left(\sum_{i=1}^{m}\lambda_{i}T_{\mu_{0}}^{\mu_{i}}\right)\#\mu_{0}, \left(\sum_{i=1}^{m}\lambda_{i}T_{\mu_{0}}^{\mu_{i},n,\varepsilon}\right)\#\mu_{0}\right)\right)\notag
\\ & +\mathbb{E}\left[\mathbb{E}\left(W_{2}\left(\left(\sum_{i=1}^{m}\lambda_{i}T_{\mu_{0}}^{\mu_{i},n,\varepsilon}\right)\#\mu_{0}, \left(\sum_{i=1}^{m}\lambda_{i}T_{\mu_{0}}^{\mu_{i},n,\varepsilon}\right)\#\tilde{\mu}_{0}^{n}\right)\bigg| \sum_{i=1}^m \lambda_i T^{\mu_i,n,\varepsilon}_{\mu_0}\right)\right]\notag \\
\le & \mathbb{E}\left(W_{2}\left(\left(\sum_{i=1}^{m}\lambda_{i}T_{\mu_{0}}^{\mu_{i}}\right)\#\mu_{0}, \left(\sum_{i=1}^{m}\lambda_{i}T_{\mu_{0}}^{\mu_{i},n,\varepsilon}\right)\#\mu_{0}\right)\right)+\mathbb{E}\left[\mathbb{E}\left(C_{\Omega,d} r_{n,d}\bigg| \sum_{i=1}^m \lambda_i T^{\mu_i,n,\varepsilon}_{\mu_0}\right)\right]\label{eqn:conditional} \\
\le & \mathbb{E}\left(\left\|\sum_{i=1}^{m}\lambda_{i}T_{\mu_{0}}^{\mu_{i}}-\sum_{i=1}^{m}\lambda_{i}T_{\mu_{0}}^{\mu_{i},n,\varepsilon}\right\|_{L^{2}(\mu_{0})}\right)+C_{\Omega,d}r_{n,d}\notag \\
\le & \sum_{i=1}^{m}\lambda_{i}\mathbb{E}\|T_{\mu_{0}}^{\mu_{i}}-T_{\mu_{0}}^{\mu_{i},n,\varepsilon}\|_{L^{2}(\mu_{0})}+C_{\Omega,d}r_{n,d}\notag\\
\lesssim &\max_{1\le i\le m}\sqrt{(1+I_{0}(\mu_{0},\mu_{i}))n^{-\frac{(\bar{\alpha} + 1)}{2(d + \bar{\alpha} + 1)}}\log n}+r_{n,d},\notag
\end{align}
where in (\ref{eqn:conditional}), we applied Theorem \ref{thm:wasserstein law of large numbers} to $\mathbb{E}\left(W_{2}\left(\left(\sum_{i=1}^{m}\lambda_{i}T_{\mu_{0}}^{\mu_{i},n,\varepsilon}\right)\#\mu_{0}, \left(\sum_{i=1}^{m}\lambda_{i}T_{\mu_{0}}^{\mu_{i},n,\varepsilon}\right)\#\tilde{\mu}_{0}^{n}\right)\bigg| \sum_{i=1}^m \lambda_i T^{\mu_i,n,\varepsilon}_{\mu_0}\right)$, which applies as $\tilde{\mu}_0^n$ and $T_{\mu_0}^{\mu_i,n,\varepsilon}$ are independent,  and hence $\left(\sum_{i=1}^{m}\lambda_{i}T_{\mu_{0}}^{\mu_{i},n,\varepsilon}\right)\#\tilde{\mu}_{0}^{n}$ is supported on an i.i.d. sample from  $\left(\sum_{i=1}^{m}\lambda_{i}T_{\mu_{0}}^{\mu_{i},n,\varepsilon}\right)\#\mu_{0}$.
\end{proof}

\subsection{Proof of Theorem \ref{thm : estimate of Aij lbcm}}
\begin{proof}
Throughout, unless otherwise noted, the expectation is with respect to all of $X_1,...,X_{2n}, Y_1^{1},...,Y_n^{1}, Y_1^{2},...,Y_{n}^{2}$ and $Z_1,...,Z_n$. We also use the notations $X^n$ to denote the set of variables $(X_1,...,X_n)$, and similarly for $Y^n$ and $Z^n$.  

We can assume without loss of generality that $0 \in \Omega$; see \cite{peyre2020computational} Remark 2.19 for invariance of $T_\eta^{\mu_1},T_\eta^{\mu_2}$ under translation. For the translation invariance of $T_{\eta}^{\mu_{1},n,\varepsilon},T_{\eta}^{\mu_{2},n,\varepsilon}$, observe that both obtaining $g_\varepsilon$, and evaluating $T_{\eta}^{\mu_{1},n,\varepsilon},T_{\eta}^{\mu_{2},n,\varepsilon}$ at fixed points, requires only the distances $\norm{X_i - Y_j}_2^2 = \norm{(X_i - t) - (Y_j - t)}_2^2$.  We calculate:

\begin{align}\label{eqn:BCM_BigSplit}
&\mathbb{E}\left [ \left | \int \langle T_{\mu_{0}}^{\mu_1} - T_{\mu_{0}}^{\eta}, T_{\mu_{0}}^{\mu_2} - T_{\mu_{0}}^{\eta} \rangle d\mu_{0} - \frac{1}{n} \sum_{i=n+1}^{2n} \langle T_{\mu_{0}}^{\mu_{1},n,\varepsilon}(X_i) - T_{\mu_{0}}^{\eta,n,\varepsilon}(X_i), T_{\mu_{0}}^{\mu_{2},n,\varepsilon}(X_i) - T_{\mu_{0}}^{\eta,n,\varepsilon}(X_i) \rangle \right | \right ]\notag  \\
= &\mathbb{E}\bigg[ \bigg | \int \langle T_{\mu_{0}}^{\mu_1} - T_{\mu_{0}}^{\eta}, T_{\mu_{0}}^{\mu_2} - T_{\mu_{0}}^{\eta} \rangle d\mu_{0} - \frac{1}{n}\sum_{i=n+1}^{2n} \langle T_{\mu_{0}}^{\mu_1}(X_i) - T_{\mu_{0}}^{\eta}(X_i), T_{\mu_{0}}^{\mu_{2}}(X_i) - T_{\mu_{0}}^{\eta}(X_i)\rangle\notag \\      
& + \frac{1}{n}\sum_{i=n+1}^{2n} \langle T_{\mu_{0}}^{\mu_1}(X_i) - T_{\mu_{0}}^{\eta}(X_i), T_{\mu_{0}}^{\mu_{2}}(X_i) - T_{\mu_{0}}^{\eta}(X_i)\rangle\notag
 \\
&-\frac{1}{n} \sum_{i=n+1}^{2n} \langle T_{\mu_{0}}^{\mu_{1},n,\varepsilon}(X_i) - T_{\mu_{0}}^{\eta,n,\varepsilon}(X_i), T_{\mu_{0}}^{\mu_{2},n,\varepsilon}(X_i) - T_{\mu_{0}}^{\eta,n,\varepsilon}(X_i) \rangle\bigg|\bigg]\notag \\
\le &\mathbb{E}\bigg[ \bigg | \int \langle T_{\mu_{0}}^{\mu_1} - T_{\mu_{0}}^{\eta}, T_{\mu_{0}}^{\mu_2} - T_{\mu_{0}}^{\mu_{2}} \rangle d\mu_{0} - \frac{1}{n}\sum_{i=n+1}^{2n} \langle T_{\mu_{0}}^{\mu_1}(X_i) - T_{\mu_{0}}^{\eta}(X_i), T_{\mu_{0}}^{\mu_{2}}(X_i) - T_{\mu_{0}}^{\eta}(X_i)\rangle\bigg|\bigg] \\  
& + \mathbb{E}\bigg[\bigg|\frac{1}{n}\sum_{i=n+1}^{2n} \langle T_{\mu_{0}}^{\mu_1}(X_i) - T_{\mu_{0}}^{\eta}(X_i), T_{\mu_{0}}^{\mu_{2}}(X_i) - T_{\mu_{0}}^{\eta}(X_i)\rangle \notag\\
&-\frac{1}{n} \sum_{i=n+1}^{2n} \langle T_{\mu_{0}}^{\mu_{1},n,\varepsilon}(X_i) - T_{\mu_{0}}^{\eta,n,\varepsilon}(X_i), T_{\mu_{0}}^{\mu_{2},n,\varepsilon}(X_i) - T_{\mu_{0}}^{\eta,n,\varepsilon}(X_i) \rangle\bigg|\bigg]\notag
    \end{align}

So, it suffices to bound two terms in (\ref{eqn:BCM_BigSplit}) separately.  To bound the first term, let $h(x)\triangleq\langle T_{\mu_{0}}^{\mu_{1}}(x)-T_{\mu_{0}}^{\eta}(x), T_{\mu_{0}}^{\mu_{2}}(x)-T_{\mu_{0}}^{\eta}(x)\rangle$.  Note that for all $x\in\Omega$, $|h(x)|\le 4|\Omega|^{2}$, and thus $Var(h(x))\le 16|\Omega|^{4}$.  From this it follows
    \begin{align*}
        &\mathbb{E}\left[ \left | 
            \int \langle T_{\mu_{0}}^{\mu_1} - T_{\mu_{0}}^{\eta}, T_{\mu_{0}}^{\mu_2} - T_{\mu_{0}}^{\eta} \rangle d\eta - \frac{1}{n}\sum_{i=n+1}^{2n} \langle T_{\mu_{0}}^{\mu_1}(X_i) - T_{\mu_{0}}^{\eta}(X_i), T_{\mu_{0}}^{\mu_{2}}(X_i) - T_{\mu_{0}}^{\eta}(X_i)\rangle \right| \right ] \\
            = &\mathbb{E} \left[\left|\mathbb{E}_{X \sim \eta}[h(X)] - \frac{1}{n}\sum_{i=n+1}^{2n} h(X_i)\right|\right] \\
            \leq &\sqrt{\frac{\text{Var}[h(X)]}{n}} \leq \sqrt{\frac{16|\Omega|^4}{n}}  \lesssim \frac{1}{\sqrt{n}}
    \end{align*}
    where we have used that for all i.i.d. random variables $U,U_1,...,U_n$ with finite variance 
    $$\mathbb{E}[|\mathbb{E}[U] - \frac{1}{n}\sum_{i=1}^n U_i|] \leq \sqrt{\frac{\text{Var}[U]}{n}}.$$

    We now proceed to the second term in (\ref{eqn:BCM_BigSplit}).
    \begin{align*}
        &\mathbb{E}\bigg[\bigg|\frac{1}{n}\sum_{i=n+1}^{2n} \langle T_{\mu_{0}}^{\mu_1}(X_i) - T_{\mu_{0}}^{\eta}(X_i), T_{\mu_{0}}^{\mu_{2}}(X_i) - T_{\mu_{0}}^{\eta}(X_i)\rangle- \\
        &\frac{1}{n} \sum_{i=n+1}^{2n} \langle T_{\mu_{0}}^{\mu_{1},n,\varepsilon}(X_i) - T_{\mu_{0}}^{\eta,n,\varepsilon}(X_i), T_{\mu_{0}}^{\mu_{2},n,\varepsilon}(X_i) - T_{\mu_{0}}^{\eta,n,\varepsilon}(X_i) \rangle\bigg|\bigg]\\
        \le &\frac{1}{n}\sum_{i=n+1}^{2n} \mathbb{E}\bigg[\bigg|\langle T_{\mu_{0}}^{\mu_1}(X_i) - T_{\mu_{0}}^{\eta}(X_i), T_{\mu_{0}}^{\mu_{2}}(X_i) - T_{\mu_{0}}^{\eta}(X_i)\rangle- \langle T_{\mu_{0}}^{\mu_{1},n,\varepsilon}(X_i) - T_{\mu_{0}}^{\eta,n,\varepsilon}(X_i), T_{\mu_{0}}^{\mu_{2},n,\varepsilon}(X_i) - T_{\mu_{0}}^{\eta,n,\varepsilon}(X_i) \rangle\bigg|\bigg]\\
        =&\mathbb{E}\bigg[\bigg|\langle T_{\mu_{0}}^{\mu_1}(X_{n+1}) - T_{\mu_{0}}^{\eta}(X_{n+1}), T_{\mu_{0}}^{\mu_{2}}(X_{n+1}) - T_{\mu_{0}}^{\eta}(X_{n+1})\rangle-\\
        &\langle T_{\mu_{0}}^{\mu_{1},n,\varepsilon}(X_{n+1}) - T_{\mu_{0}}^{\eta,n,\varepsilon}(X_{n+1}), T_{\mu_{0}}^{\mu_{2},n,\varepsilon}(X_{n+1}) - T_{\mu_{0}}^{\eta,n,\varepsilon}(X_{n+1}) \rangle\bigg|\bigg]\\
        =&\mathbb{E}\bigg[\bigg|\langle (T_{\mu_{0}}^{\mu_{1}}(X_{n+1})-T_{\mu_{0}}^{\eta}(X_{n+1}))-(T_{\mu_{0}}^{\mu_{1},n,\varepsilon}(X_{n+1})-T_{\mu_{0}}^{\eta,n,\varepsilon}(X_{n+1})), T_{\mu_{0}}^{\mu_{2}}(X_{n+1})-T_{\mu_{0}}^{\eta}(X_{n+1})\rangle\\
        &-\langle T_{\mu_{0}}^{\mu_{1},n,\varepsilon}(X_{n+1})-T_{\mu_{0}}^{\eta,n,\varepsilon}(X_{n+1}), (T_{\mu_{0}}^{\mu_{2},n,\varepsilon}(X_{n+1})-T_{\mu_{0}}^{\eta,n,\varepsilon}(X_{n+1}))-(T_{\mu_{0}}^{\mu_{2}}(X_{n+1})-T_{\mu_{0}}^{\eta}(X_{n+1}))\rangle\bigg|\bigg]\\
        \le&\mathbb{E}\bigg[\bigg|\langle (T_{\mu_{0}}^{\mu_{1}}(X_{n+1})-T_{\mu_{0}}^{\eta}(X_{n+1}))-(T_{\mu_{0}}^{\mu_{1},n,\varepsilon}(X_{n+1})-T_{\mu_{0}}^{\eta,n,\varepsilon}(X_{n+1})), T_{\mu_{0}}^{\mu_{2}}(X_{n+1})-T_{\mu_{0}}^{\eta}(X_{n+1})\rangle\bigg|\bigg]\\
        +&\mathbb{E}\bigg[\bigg|\langle T_{\mu_{0}}^{\mu_{1},n,\varepsilon}(X_{n+1})-T_{\mu_{0}}^{\eta,n,\varepsilon}(X_{n+1}), (T_{\mu_{0}}^{\mu_{2},n,\varepsilon}(X_{n+1})-T_{\mu_{0}}^{\eta,n,\varepsilon}(X_{n+1}))-(T_{\mu_{0}}^{\mu_{2}}(X_{n+1})-T_{\mu_{0}}^{\eta}(X_{n+1}))\rangle\bigg|\bigg]\\
        \le&\mathbb{E}\bigg[\|T_{\mu_{0}}^{\mu_{1}}(X_{n+1})-T_{\mu_{0}}^{\eta}(X_{n+1})-(T_{\mu_{0}}^{\mu_{1},n,\varepsilon}(X_{n+1})-T_{\mu_{0}}^{\eta,n,\varepsilon}(X_{n+1}))\|\bigg]\mathbb{E}\bigg[\|T_{\mu_{0}}^{\mu_{2}}(X_{n+1})-T_{\mu_{0}}^{\eta}(X_{n+1})\|\bigg]\\
        +&\mathbb{E}\bigg[\| T_{\mu_{0}}^{\mu_{1},n,\varepsilon}(X_{n+1})-T_{\mu_{0}}^{\eta,n,\varepsilon}(X_{n+1})\|\bigg] \mathbb{E}\bigg[\|(T_{\mu_{0}}^{\mu_{2},n,\varepsilon}(X_{n+1})-T_{\mu_{0}}^{\eta,n,\varepsilon}(X_{n+1}))-(T_{\mu_{0}}^{\mu_{2}}(X_{n+1})-T_{\mu_{0}}^{\eta}(X_{n+1}))\|\bigg]\\
        \le &|\Omega|\mathbb{E}\bigg[\|T_{\mu_{0}}^{\mu_{1}}(X_{n+1})-T_{\mu_{0}}^{\eta}(X_{n+1})-(T_{\mu_{0}}^{\mu_{1},n,\varepsilon}(X_{n+1})-T_{\mu_{0}}^{\eta,n,\varepsilon}(X_{n+1}))\|\bigg]\\
        +&|\Omega|\mathbb{E}\bigg[\|(T_{\mu_{0}}^{\mu_{2},n,\varepsilon}(X_{n+1})-T_{\mu_{0}}^{\eta,n,\varepsilon}(X_{n+1}))-(T_{\mu_{0}}^{\mu_{2}}(X_{n+1})-T_{\mu_{0}}^{\eta}(X_{n+1}))\|\bigg]\\
        \le&|\Omega|\bigg(\mathbb{E}\bigg[\| T_{\mu_{0}}^{\mu_{1}}(X_{n+1})-T_{\mu_{0}}^{\mu_{1},n,\varepsilon}(X_{n+1})\|\bigg] +\mathbb{E}\bigg[\|T_{\mu_{0}}^{\eta}(X_{n+1})-T_{\mu_{0}}^{\eta,n,\varepsilon}(X_{n+1})\|\bigg]\\
        +&\mathbb{E}\bigg[\| T_{\mu_{0}}^{\mu_{2}}(X_{n+1})-T_{\mu_{0}}^{\mu_{2},n,\varepsilon}(X_{n+1})\|\bigg] +\mathbb{E}\bigg[\|T_{\mu_{0}}^{\eta}(X_{n+1})-T_{\mu_{0}}^{\eta,n,\varepsilon}(X_{n+1})\|\bigg]\bigg).
    \end{align*}
Each of these four expectations can be controlled by applying Jensen's inequality followed by Theorem \ref{thm : pooladian entropic estimation} to give the bound $\lesssim \sqrt{n^{-\frac{\bar{\alpha}+1}{2(d+\bar{\alpha}+1)}}\log(n)}=n^{-\frac{\bar{\alpha}+1}{4(d+\bar{\alpha}+1)}}\sqrt{\log(n)}.$  The result follows.
\end{proof}

\subsection{Proof of Corollary \ref{cor : estimate lambda lbcm}}

\begin{proof}
The proof is similar to one found in \cite{werenski2022measure} and we include it here for completeness. Let $B_{n}$ denote the entrywise bound in Theorem \ref{thm : estimate of Aij lbcm}. Noting that $\hat{\lambda}^{T}\hat{A}^{L}\hat{\lambda} \leq  \lambda_{*}^{T}\hat{A}^{L}\lambda_{*}$ by construction, we estimate 
    \begin{align*}
        \mathbb{E}[\hat{\lambda}^{T}A^{L}\hat{\lambda}] 
        &= \mathbb{E}\left[\hat{\lambda}^{T}(A^{L}-\hat{A}^{L})\hat{\lambda}\right]+\mathbb{E}[\hat{\lambda}^{T}\hat{A}^{L}\hat{\lambda}] \\
        &\leq \mathbb{E}\left[|\hat{\lambda}^{T}(A^{L}-\hat{A}^{L})\hat{\lambda}|\right] +\mathbb{E}[\lambda_{*}^{T}\hat{A}^{L}\lambda_{*}] \\
        &= \mathbb{E}\left[|\hat{\lambda}^{T}(A^{L}-\hat{A}^{L})\hat{\lambda}|\right]+\mathbb{E}[\lambda_{*}^{T}(\hat{A}^{L}-A^{L})\lambda_{*}] \\
        &= \mathbb{E}\left[\bigg| \sum_{i,j=1}^m (\hat{\lambda})_{i}(\hat{\lambda})_{j} (A^{L} - \hat{A}^{L})_{ij}\bigg|\right]+
            \mathbb{E}\left[\bigg| \sum_{i,j=1}^m (\lambda_{*})_{i}(\lambda_{*})_{j} (A^{L} - \hat{A^{L}})_{ij}\bigg|\right] \\
        &\leq \mathbb{E}\left[ \sum_{i,j=1}^m (\hat{\lambda})_{i}(\hat{\lambda})_{j} |A_{ij}^{L} - \hat{A}_{ij}^{L}|\right] + 
            \mathbb{E}\left[\sum_{i,j=1}^m (\lambda_{*})_{i}(\lambda_{*})_{j} |A_{ij}^{L} - \hat{A}_{ij}^{L}|\right] \\
        &\leq 2\mathbb{E}\left[\sum_{i,j=1}^m |A_{ij}^{L} - \hat{A}_{ij}^{L}|\right] = 2\sum_{i,j=1}^m \mathbb{E}[|A_{ij}^{L} - \hat{A}_{ij}^{L}|] \lesssim 2m^{2}B_{n} 
    \end{align*}
    In the second line we have used the assumption that $\lambda_*^TA^{L}\lambda_* = 0$. The second to last line uses the triangle inequality. The last line uses $\lambda_*, \hat{\lambda} \in \Delta^m$ so their entries in $[0,1]$ and then concludes with Theorem \ref{thm : estimate of Aij lbcm}.
    
    Since $A^{L}$ is positive semidefinite and by assumption $\lambda_{*}\in\Delta^{m}$ satisfies $\lambda_{*}^{T}A^{L}\lambda_{*}=0$, it follows that $\lambda_{*}$ is an eigenvector of $A^{L}$ with eigenvalue 0.  Let $0<a_{2}\le\dots\le a_{m}$ be the non-zero eigenvalues of $A^{L}$ with associated orthonormal eigenvectors $v_{2},\dots,v_{m}$. Orthogonally decompose $\hat{\lambda}=\hat{\beta}\lambda_{*}+\hat{\lambda}_{\perp}$, where $\hat{\beta}\in \mathbb{R}$ and $\hat{\lambda}_{\perp}$ is in the span of $\{v_{2},\dots,v_{m}\}$.  Note that $\hat{\beta}$ and $\hat{\lambda}_{\perp}$ are random.  Then,
    
    \begin{align*}&\mathbb{E}[\|\hat{\lambda}-\hat{\beta}\lambda_{*}\|_{2}^2]\\ 
    =&\mathbb{E}[\|\hat{\lambda}_{\perp}\|_{2}^2]\\
    =&\mathbb{E}\left[\sum_{i=2}^{m}|v_{i}^{T}\hat{\lambda}_{\perp}|^{2}\right]\\
    \leq &\frac{1}{a_{2}}\mathbb{E}\left[\sum_{i=2}^{m}a_{i}|v_{i}^{T}\hat{\lambda}_{\perp}|^{2}\right]\\
    =&\frac{1}{a_{2}}\mathbb{E}[|(\hat{\lambda}_{\perp})^{T}A^{L}\hat{\lambda}_{\perp}|]\\
    =&\frac{1}{a_{2}}\mathbb{E}[|\hat{\lambda}^{T}A^{L}\hat{\lambda}|]\\
    \lesssim &\frac{2 m^{2}}{a_{2}}B_{n}.
    \end{align*}
    
    Summing both sides of the equation $\hat{\lambda} = \hat{\beta} \lambda_*  + \hat{\lambda}_{\perp}$ and recalling $\lambda_{*}, \hat{\lambda}\in\Delta^{m}$ yields 
    \begin{align*}1 &= \hat{\beta} +\sum_{j=1}^{m}(\hat{\lambda}_\perp)_{j} \leq \hat{\beta}+\|\hat{\lambda}_\perp\|_{1} \leq  \hat{\beta}+\sqrt{m}\|\hat{\lambda}_\perp\|_{2},
    \end{align*}
    which implies that $\mathbb{E}[(1 - \hat{\beta})^2] \leq m \mathbb{E}[\| \hat{\lambda}_\perp\|_2^2] \lesssim \frac{2m^3}{a_2} B_n$.
    
    Finally, we use the fact that $\hat{\lambda}-\hat{\beta}\lambda_{*}=\hat{\lambda}_{\perp}$ and $(\hat{\beta}-1)\lambda_{*}$ are orthogonal to bound: 
    
    \begin{align*}
        \mathbb{E}[\|\hat{\lambda}-\lambda_{*}\|_{2}^{2}] 
        &= \mathbb{E}[\|\hat{\lambda}-\hat{\beta}\lambda_{*}\|_{2}^{2}+\|(\hat{\beta}-1)\lambda_{*}\|_{2}^{2}] \\
        &= \mathbb{E}[\|\hat{\lambda}-\hat{\beta}\lambda_{*}\|_{2}^{2}] + \mathbb{E}[\|(\hat{\beta}-1)\lambda_{*}\|_{2}^{2}]\\
        &\lesssim \frac{2 m^{2}}{a_{2}}B_{n}+\mathbb{E}[(\hat{\beta}-1)^{2}\|\lambda_{*}\|_{2}^{2}]\\
        &\le \frac{2 m^{2}}{a_{2}}B_{n} + \mathbb{E}[(\hat{\beta}-1)^{2}]\\
        &\lesssim \frac{2 m^{2}}{a_{2}}B_{n} + \frac{2m^3}{a_2} B_n
    \end{align*}as desired. 
\end{proof}

\begin{remark}

As shown in the proof, the implicit constant in Corollary \ref{cor : estimate lambda lbcm} depends on the second eigenvalue $a_{2}$ of $A^{L}$ as $1/a_{2}$.  The choice of $\mu_0$ may be especially important to ensuring that the second eigenvalue is not too small.
\end{remark}

\section{Proofs From Section \ref{sec:RepresentationalCapacity}}
\label{sec:ProofsSection5}
\subsection{Proof of Proposition \ref{prop : bcm on gaussians}}

\begin{proof}
    That an element of $\BCM(\{\mu_i\}_{i\in I})$ is a Gaussian is a consequence of Theorem 3.10 in \cite{alvarez2018wide}. That $\LBCM(\{\mu_i\}_{i\in I}; \mu_0)$ is a Gaussian follows from the fact $T_{\mu_{0}}^{\mu_{i}}$ is an affine transformation and therefore $\int \lambda T_{\mu_{0}}^{\mu_{i}} d\lambda(i)$ is also affine, and the family of Gaussian measures is closed under affine transformations. The latter fact is well-known, see \cite{goodfellow2016deep}. 
\end{proof}

\subsection{Proof of Theorem \ref{thm : lbcm 1D capacity}}

We begin by establishing a preliminary result. This serves the role of a Choquet-type theorem in that it identifies the set of extreme points in a convex set and shows that every point in the set is a convex combination of the extreme points.

\begin{proposition} \label{prop : choquet type}
    Let $\mathcal{T}([0,1])$ denote the set of functions 
    \begin{equation*}
        \mathcal{T}([0,1]) \triangleq \{ T : [0,1] \rightarrow [0,1] \ | \ T \text{ is increasing} \}.
    \end{equation*}
    The set $\mathcal{T}([0,1])$ is convex. Let $T_a(x) \triangleq \pmb{1}[x \geq a]$. Then for every $T \in \mathcal{T}([0,1])$ there is a measure $\lambda \in \mathcal{P}([0,1])$ such that for $U([0,1])$-almost every $x$ it holds
    \begin{equation*}
        T(x) = \int T_a(x) d\lambda(a).
    \end{equation*}
\end{proposition}

\begin{proof}
    Let $T \in \mathcal{T}([0,1])$. Since $T$ is monotonic at each point in the domain of $T$ there exists well-defined left and right limit. We can define a right continuous version of $T$, say $T_+$ for every $x \in [0,1)$ by
    \begin{equation*}
        T_+(x) = \lim_{y \rightarrow x^+} T(y)
    \end{equation*}
    and set $T_+(1) = T(1)$.
    
    Since $T$ is monotonic it has only countably many discontinuities and $T_+$ agrees with $T$, except possibly at these discontinuities, the functions $T$ and $T_+$ disagree on a countable set which is of measure zero.  Now define the measure $\lambda$ by 
    \begin{equation*}
        \lambda([0,x]) = T_+(x), \ x \in [0,1) \hspace{1cm} \lambda(\{1\}) = 1 - T_+(1).
    \end{equation*}
    One can show that this satisfies all the required properties to be a probability measure over $[0,1]$.
    
    For $x \in [0,1)$ we have
    \begin{equation*}
        \int_{[0,1]} \pmb{1}[a \leq x] d\lambda(a) = \int_{[0,x]} 1 d\lambda(a) = \lambda([0,x]) = T_+(x).
    \end{equation*}Since $T_+$ agrees with $T$ for almost every $x$ and the set $\{1\}$ is of measure zero we have established that for almost every $x$
    \begin{equation*}
        \int_{[0,1]} \pmb{1}[a \leq x] d\lambda(a) = T(x)
    \end{equation*}
    and therefore the proposed $\lambda$ satisfies the claim.
    
    Since $T$ was chosen arbitrarily in $\mathcal{T}([0,1])$, the result holds over the entire class of functions.
\end{proof}

We will also leverage the following two known results from \cite{panaretos2020invitation}.

\begin{theorem}[\cite{panaretos2020invitation} Theorem 3.1.5] \label{thm : barycenter stability}
    Let $m \in \mathbb{N}$ and let $\lambda = (1/m,...,1/m)$. Suppose that $W_2(\mu^k_i,\mu_i) \rightarrow 0$ for $i=1,...,m$ and let $\mu^k_\lambda$ be the Wasserstein barycenter with parameter $\lambda$ for $(\mu^k_1,...,\mu^k_m)$. Then any limit point of the sequence $\{\mu_\lambda^k\}_{k=1}^\infty$ is the Wasserstein barycenter of $\mu_1,...,\mu_m$ with coordinate $\lambda$.
\end{theorem}

\begin{corollary} \label{cor : barycenter stability}
    Let $m \in \mathbb{N}$ and let $\lambda \in \mathbb{Q}^m \cap \Delta^m$. Suppose that $W_2(\mu^k_i,\mu_i) \rightarrow 0$ for $i=1,...,m$ and let $\mu^k_\lambda$ be the Wasserstein barycenter with parameter $\lambda.$ Then any limit point of the sequence $\{\mu_\lambda^k\}_{k=1}^\infty$ is the Wasserstein barycenter of $\mu_1,...,\mu_m$ with coordinate $\lambda$.
\end{corollary}
\begin{proof}
    This follows from Theorem \ref{thm : barycenter stability} by repeating the measures $\mu_i$ as necessary. 
\end{proof}

\begin{theorem}[\cite{panaretos2020invitation} Theorem 3.1.9] \label{thm : barycenter compatible}
    Let $m \in \mathbb{N}$ and let $\lambda = (1/m,...,1/m)$. Suppose that $\mu_0,...,\mu_m$ are compatible. Then the Wasserstein barycenter $\mu_\lambda$ is given by
    \begin{equation*}
        \mu_\lambda = \left ( \frac{1}{m}\sum_{i=1}^m T_i \right )\#\mu_0,
    \end{equation*}
    where $T_i$ is the optimal transport map from $\mu_0$ to $\mu_i$.
\end{theorem}

\begin{corollary} \label{cor : barycenter compatible}
    Let $m \in \mathbb{N}$ and let $\lambda \in \mathbb{Q}^m \cap \Delta^m$. Suppose that $\mu_0,...,\mu_m$ are compatible. Then the Wasserstein barycenter $\mu_\lambda$ is given by
    \begin{equation*}
        \mu_\lambda = \left ( \sum_{i=1}^m \lambda_i T_i \right )\#\mu_0,
    \end{equation*}
    where $T_i$ is the optimal transport map from $\mu_0$ to $\mu_i$.
\end{corollary}
\begin{proof}
    This follows from Theorem \ref{thm : barycenter compatible} by repeating the measures $\mu_i$ as necessary. 
\end{proof}

We can now proceed to the main proof.

\begin{proof}[Proof of Theorem \ref{thm : lbcm 1D capacity}]
    To each measure $\mu \in \mathcal{P}([0,1])$ associate to it the optimal transport map from $U([0,1])$ to $\mu$, denoted by $T$. The map $T$ is almost everywhere uniquely defined and is contained in the sub-differential of a convex function. In addition $T$ is an increasing function and the image of $[0,1]$ under $T$ is contained in $[0,1]$.

    By Proposition \ref{prop : choquet type} there exists a measure $\lambda \in \mathcal{P}([0,1])$ such that for $U([0,1])$ almost every $x$
    \begin{equation*}
        T(x) = \int_{[0,1]} \pmb{1}[a \leq x] d\lambda(a).
    \end{equation*}
    Note that the measures $\{ a \delta_0 + (1-a)\delta_1 \ | \ a \in [0,1] \}$ have optimal transport maps given precisely by $\{ x \mapsto \pmb{1}[x \geq a] \ | \ a \in [0,1] \}.$ Denote $a \delta_0 + (1-a)\delta_1$ by $\mu_a$ and $T_a$ the corresponding transport map from $U([0,1])$ to $\mu_a$. With this notation it holds for almost every $x$ that
    \begin{equation} \label{eq : ae equality}
        T(x) = \int_{[0,1]} T_a(x) d\lambda(a).
    \end{equation}
    This implies that for every bounded continuous function $f$ we have
    \begin{align*}
        \int f(x) d\mu(x) 
        &= \int f(T(x)) dU([0,1])(x) \\
        &= \int f \left ( \int_{[0,1]} T_a(x) d\lambda(a) \right ) dU([0,1])(x) \\
        &= \int f(x) d \left [ \left ( \int_{[0,1]} T_a(x) d\lambda(a) \right )\#U([0,1])\right ](x)
    \end{align*}
    where the first equality holds because $T\#U([0,1]) = \mu$, the second holds because of \eqref{eq : ae equality}, and the third holds by definition of the push-forward of a measure. This shows that with respect to weak convergence the measure $\mu = \left ( \int_{[0,1]} T_a(x) d\lambda(a) \right )\#U([0,1])$ and therefore is contained in $\LBCM(U([0,1]); \{a\delta_0 + (1-a)\delta_1 \ | \ a \in [0,1] \})$.  Since this $\mu$ was arbitrary we see that every measure in $\mathcal{P}([0,1])$ is indeed contained in $\LBCM(U([0,1]); \{a\delta_0 + (1-a)\delta_1 \ | \ a \in [0,1] \})$. 
    
    For the weak convergence of $\BCM(\{a\delta_0 + (1-a)\delta_1 \ | \ a \in [0,1] \})$ we would like to appeal to compatibility but cannot immediately do so because the measures do not even have optimal transport maps. However, one can subvert this problem by an approximation argument and using Corollaries \ref{cor : barycenter stability} and \ref{cor : barycenter compatible}.  We will establish that the set 
    \begin{equation} \label{eq : rational measures}
        \left \{ \sum_{i=1}^m a_i\delta_{b_i} \ | \ m \in \mathbb{N}, a_i,b_i \in \mathbb{Q} \cap [0,1], \sum_{i=1}^m a_i = 1 \right \} \subseteq \ \BCM(\{a\delta_0 + (1-a)\delta_1 \ | \ a \in [0,1] \})
    \end{equation}
    and the former set is clearly dense with respect to weak convergence. 

    We introduce the family of measures parameterized by $a \in [0,1], b \in (0,1/2)$:
    \begin{equation*}
        \nu_{a,b} = a U([0,b]) + (1-a)U([1-b, b]).
    \end{equation*}
    The measures $\nu_{a,b}$ are absolutely continuous and satisfy
    \begin{equation}
        \lim_{b \rightarrow 0^+} W_2\left (a\delta_0 + (1-a)\delta_1, \nu_{a,b} \right ) = 0. \label{eq : smooth to diracs}
    \end{equation}
    In addition, the optimal transport map from $U([0,1])$ to $\nu_{a,b}$, which we denote by $T_{a,b}$ for $a \in (0,1)$, is given by
    \begin{equation*}
        T_{a,b}(x) = \begin{cases}
            \frac{b}{a}x & x \leq a \\
            (1 - b) + \frac{b}{1-a}(x-a)& x > a
        \end{cases}
    \end{equation*}
    while $T_{1,b} = bx$ and $T_{0,b} = 1 - b + bx$. These maps satisfy for $a > 0$
    \begin{equation*}
        \lim_{b \rightarrow 0^+} T_{a,b} = \pmb{1}[x > a]
    \end{equation*}
    as well as the bound
    \begin{equation*}
        \max_{x \in [0,1]} |T_{a,b}(x) - \pmb{1}[x > a]| \leq b.
    \end{equation*}
   For $a = 0$ they satisfy
    \begin{equation*}
        \lim_{b \rightarrow 0^+} T_{0,b} = 1
    \end{equation*}
    as well as the bound
    \begin{equation*}
        \max_{x \in [0,1]} |T_{0,b}(x) - 1| \leq b.
    \end{equation*}
    This bound further implies for any $(\lambda_0,...,\lambda_{m}) \in \Delta^{m+1}$
    \begin{align*}
        \max_{x \in [0,1]} \left | \left (\sum_{i=0}^{m} \lambda_i T_{i/m,b}(x) \right ) - \left ( \sum_{i=1}^{m} \lambda_i \pmb{1}[x > i/m] + \lambda_0 \cdot 1 \right ) \right | 
        &\leq \max_{x \in [0,1]} \lambda_0 |T_{0,b}(x) - 1| + \sum_{i=1}^{m} \lambda_i |T_{i/m,b}(x) - \pmb{1}[x > i/m]| \\
        &\leq \max_{x \in [0,1]} \lambda_0 b + \sum_{i=1}^{m} \lambda_i b = b.
    \end{align*}
    This bound further implies for any $(\lambda_0,...,\lambda_{m}) \in \Delta^{m+1}$ that
    \begin{align}
        &W_2^2 \left ( 
            \left ( \sum_{i=0}^{m} \lambda_i T_{i/m,b} \right )\#U([0,1]), 
            \left ( \sum_{i=1}^{m} \lambda_i \pmb{1}[x > i/m] + \lambda_0 \cdot 1 \right )\#U([0,1]) 
        \right ) \notag \\
        &\leq \int \left ( 
            \left ( \sum_{i=0}^{m} \lambda_i T_{i/m,b} \right ) -  
            \left ( \sum_{i=1}^{m} \lambda_i \pmb{1}[x > i/m] + \lambda_0 \cdot 1 \right )
        \right )^2 dU([0,1])(x) \notag \\
        &\leq \int b^2 dU([0,1]) = b^2 \label{eq : step function bound}
    \end{align}
    where the first inequality is realized by the choose of coupling given by
    \begin{equation*}
        \pi = \left (
            \left ( \sum_{i=0}^{m} \lambda_i T_{i/m,b} \right ),
            \left ( \sum_{i=1}^{m} \lambda_i \pmb{1}[x > i/m] + \lambda_0 \cdot 1 \right )
        \right )\#U([0,1]).
    \end{equation*}
    Next observe that because $U([0,1])$ and $\nu_{a,b}$ for $b > 0$ are absolutely continuous and one-dimensional they are compatible and therefore by Corollary \ref{cor : barycenter compatible} for $(\lambda_0,...,\lambda_{m}) \in \mathbb{Q}^{m+1} \cap \Delta^{m+1}$ it holds that the Wasserstein barycenter with coordinate $\lambda$ of $(\nu_{0,b}, \nu_{1/m,b},...,\nu_{(m-1)/m,b},\nu_{1,b})$ is given by
    \begin{equation*}
        \left ( \sum_{i=0}^{m} \lambda_i T_{i/m,b} \right )\#U([0,1]).
    \end{equation*}
    We are now able to apply Corollary \ref{cor : barycenter stability} to $\lambda \in \mathbb{Q}^{m+1} \cap \Delta^{m+1}$. This is done with $\mu_i^k = \nu_{i/m, 1/k}$ and $\mu_i = (i/m)\delta_0 + (1 - i/m)\delta_1$. In \eqref{eq : smooth to diracs} we have established that 
    \begin{equation*}
        \lim_{k \rightarrow \infty} W_2(\mu_i^k,\mu_i) = \lim_{k \rightarrow \infty}  W_2(\nu_{i/m, 1/k}, (i/m)\delta_0 + (1 - i/m)\delta_1) = 0.
    \end{equation*}
    We have also established that the barycenter of $(\mu_0^k,...,\mu_m^k) = (\nu_{0/m, 1/k},...,\nu_{m/m,1/k})$ is given by
    \begin{equation*}
        \left ( \sum_{i=0}^{m} \lambda_i T_{i/m,1/k} \right )\#U([0,1]).
    \end{equation*}
    and by \eqref{eq : step function bound} 
    \begin{equation*}
         \lim_{k \rightarrow \infty }W_2 \left ( 
            \left ( \sum_{i=0}^{m} \lambda_i T_{i/m,1/k} \right )\#U([0,1]), 
            \left ( \sum_{i=1}^{m} \lambda_i \pmb{1}[x > i/m] + \lambda_0 \cdot 1 \right )\#U([0,1]) 
        \right ) \leq \lim_{k \rightarrow \infty} \frac{1}{k} = 0.
    \end{equation*}
    Therefore by Corollary \ref{cor : barycenter stability} the barycenter of $(\mu_0,...,\mu_m) = (\delta_1, (1/m)\delta_0 + (1 - 1/m)\delta_1,...,(1-1/m)\delta_0 + (1/m)\delta_1, \delta_0)$ with coordinate $\lambda$ is given by
    \begin{equation*}
        \left ( \sum_{i=1}^{m} \lambda_i \pmb{1}[x > i/m] + \lambda_0 \cdot 1 \right )\#U([0,1])
    \end{equation*}
    which we can expand as
    \begin{equation}
        \left ( \sum_{i=1}^{m} \lambda_i \pmb{1}[x > i/m] + \lambda_0 \cdot 1 \right )\#U([0,1]) = \sum_{i=1}^{m} \frac{1}{m} \delta_{\sum_{j = 0}^{i-1} \lambda_j}. \label{eq : push forward form}
    \end{equation}
    It is a straightforward exercise to show that every set in \eqref{eq : rational measures} can be expressed in the form of \eqref{eq : push forward form} by taking $m$ sufficiently large and appropriately setting $\lambda$.

    That $\normalfont{\text{conv}}(\{U([0,1])\} \cup \{a\delta_0 + (1-a)\delta_1 \ | \ a \in [0,1] \})$ is not dense follows from the fact that it only consists of measures of the form
    \begin{equation*}
        \lambda_1 \delta_0 + \lambda_2 \delta_1 + (1-\lambda_1-\lambda_2)U([0,1]),
    \end{equation*}where $(\lambda_1,\lambda_2,\lambda_3)\in\Delta^{3}$.
\end{proof}

\subsection{Proof of Proposition \ref{prop : maps to extreme are extreme}}

\begin{proof}
    That $\mathcal{V}(C)$ consists of the optimal transport maps from $U(C)$ to $\mathcal{P}(\{v_i\}_{i=1}^\ell)$ follows from the fact that for every $T \in \mathcal{V}(C)$ the measure $T\#U(C) \in \mathcal{P}(\{v_i\}_{i=1}^\ell)$ since by definition of $\mathcal{V}(C)$ it holds that $T(x) \in \{v_i\}_{i=1}^\ell$ almost surely. In addition for every $\mu \in \mathcal{P}(\{v_i\}_{i=1}^\ell)$ the optimal transport map $T$ (which must be in $\mathcal{T}(C)$ by Theorem \ref{thm : knott-smith and brenier}) from $U(C)$ to $\mu$ must be in $\mathcal{V}(C)$ since it must hold with probability 1 that $T(x) \in \{v_i\}_{i=1}^\ell$ in order for $T$ to transport $U(C)$ to $\mu$. 

    To establish that $\mathcal{V}(C)$ are extreme in $\mathcal{T}(C)$, for each $T \in \mathcal{V}(C)$ we will construct a continuous linear functional $F_T$ on $\mathcal{T}(C)$ such that 
    \begin{equation*}
        T = \argmax_{T' \in \mathcal{T}(C)} F_T(T').
    \end{equation*}
    This will demonstrate that $T$ is not only extreme but \textit{exposed} in $\mathcal{T}(C)$. 

    In order to give the function $F_T$ we will use that in a convex polytope $\text{conv}(\{v_i\}_{i=1}^\ell)$ in $\mathbb{R}^d$ for every extreme point there exists $v_i$ there exists at least one vector $u_i \in \mathbb{R}^d$ such that
    \begin{equation*}
        v_i = \argmax_{v \in \text{conv}(\{v_i\}_{i=1}^\ell)} \langle v, u_i \rangle
    \end{equation*}
    and the maximum is achieved uniquely.
    
    Now let $\{u_i\}_{i=1}^\ell$ be such a set of vectors for $\{v_i\}_{i=1}^\ell$. In addition define the sets $R_1,...,R_\ell \subset C$ according to
    \begin{equation*}
        R_i \triangleq T^{-1}(v_i).
    \end{equation*}
    These sets are disjoint and since $T(x) \in \{v_i\}_{i=1}^\ell$ almost surely it holds that
    \begin{equation*}
        \sum_{i=1}^\ell U(C)[R_i] = 1.
    \end{equation*}
    Finally define the function $M_T:\mathbb{R}^d \rightarrow \mathbb{R}^d$ by
    \begin{equation*}
        M_T(x) = \sum_{i=1}^\ell \pmb{1}[x \in R_i] u_i.
    \end{equation*}
    
    The linear functional we will use is given by
    \begin{equation*}
        F_T(T') = \int \langle T'(x), M_T(x) \rangle d[U(C)](x).
    \end{equation*}
    This is clearly linear by the linearity of inner products and integration. We must show that $T$ is the unique maximizer of this functional.

    To see this, we have for every $T' \in \mathcal{T}(C)$
    \begin{align*}
        F_T(T') 
        &= \int \langle T'(x), M_T(x) \rangle d[U(C)](x) \\
        &= \sum_{i=1}^\ell \int_{R_i} \langle T'(x), M_T(x) \rangle d[U(C)](x) \\
        &= \sum_{i=1}^\ell \int_{R_i} \langle T'(x), u_i \rangle d[U(C)](x) \\
        &\leq \sum_{i=1}^\ell \int_{R_i} \langle v_i, u_i \rangle d[U(C)](x) \\
        &= \sum_{i=1}^\ell \langle v_i,u_i \rangle U(C)[R_i]
    \end{align*}
    The inequality follows from the definition of $u_i$ and the fact that $T' \in \mathcal{T}(C)$ and therefore 
    \begin{equation*}
        T'(x) \in C \implies \langle T'(x), u_i \rangle \leq \max_{x \in C} \langle x,u_i \rangle = \langle v_i, u_i \rangle.
    \end{equation*}
    The equality case happens if and only if for every $i = 1,...,\ell$ almost every $x \in R_i$ it holds that $T'(x) = v_i = T(x)$ which demonstrates that indeed $T$ is the unique (up to almost everywhere equality) maximizer of $F_T$ over $\mathcal{T}(C)$ and is therefore exposed which implies it is extreme.

    Since $T$ was chosen arbitrarily in $\mathcal{V}(C)$ we can conclude that the entire set $\mathcal{V}(C)$ must be extreme.
\end{proof}

\subsection{Proof of Theorem \ref{thm : questions negative}}

The polytope that we use is the set
    $C_0 = \normalfont{\text{conv}} \left ( \{(0,0), (0,1), (1,0) \} \right )$
and the convex function is given by $\phi_0(x,y) = \frac{1}{4}\frac{x^2}{(2-y)}.$  One can directly compute the Hessian of this function and check that it is positive definite and also show that $\nabla \phi_{0}(C_0) \subset C_{0}$ so that $\nabla \phi_0 \in \mathcal{T}(C_0)$. To establish that $\nabla \phi_0$ is not contained in $\normalfont{\text{conv}}(\mathcal{V}(C_0))$, we must characterize the structure of the functions which are contained in the set $\mathcal{V}(C_0)$.

\begin{lemma} \label{lem : structure of V}
    A function $T$ is contained in $\mathcal{V}(C_0)$ if and only if there exist a vector $b \in \mathbb{R}^3$ such that $T(x) \in \partial v(x;b)$ where 
    \begin{equation*}
        v(x;b) = \max_{i\in \{1,2,3\}} \langle v_i, x\rangle + b_i
    \end{equation*}
    with $v_1 = (0,0), v_2 = (0,1), v_3 = (1,0)$.
\end{lemma}
\begin{proof}
    Clearly if $T(x) \in \partial v(x;b)$ for some $b \in \mathbb{R}^d$ then $T \in \mathcal{V}(C_0)$. For the reverse direction let $\mu = T\#U(C_0)$ so that $\mu = a_1 \delta_{v_1} + a_2 \delta_{v_2} + a_3 \delta_{v_3}$, which is the case because $T \in \{v_1,v_2,v_3\}$ for $U(C_0)$ by definition of $\mathcal{V}(C_0)$. Since it is also the case that $T$ is contained in the \emph{sub-differential} of a convex function by Theorem \ref{thm : knott-smith and brenier} it must be the case that $T$ is the optimal transport map from $U(C_0)$ to $\mu$. We refer to \cite{peyre2020computational} Section 5.2 for proof that this implies $T$ has the claimed form. 
\end{proof}

\begin{lemma} \label{lem : first coordinate decrease}
    Let $T \in \mathcal{V}(C_0)$. The following holds for all but at most one $a_1 \in [0,1]$. For all $0\leq a_2 < a_2' \leq 1 - a_1$ we have
    \begin{equation*}
        \langle e_1, T((a_1,a_2))\rangle \geq \langle e_1, T((a_1,a_2')) \rangle.
    \end{equation*}
\end{lemma}

\begin{proof}
 By Lemma \ref{lem : structure of V} there exists some $b \in \mathbb{R}^3$ such that $T(x) \in \partial v(x;b)$. Define the sets
    \begin{align*}
        R_1 &= \{ x \in \mathbb{R}^2 \ | \ \langle x, v_1 \rangle + b_1 = v(x;b) \} \\
        R_2 &= \{ x \in \mathbb{R}^2 \ | \ \langle x, v_2 \rangle + b_2 = v(x;b) \} \\
        R_3 &= \{ x \in \mathbb{R}^2 \ | \ \langle x, v_3 \rangle + b_3 = v(x;b) \} \\
    \end{align*}
    These are illustrated in Figure \ref{fig : lbcm structure}.
    \begin{figure}
        \centering
        \includegraphics[width=0.5\linewidth]{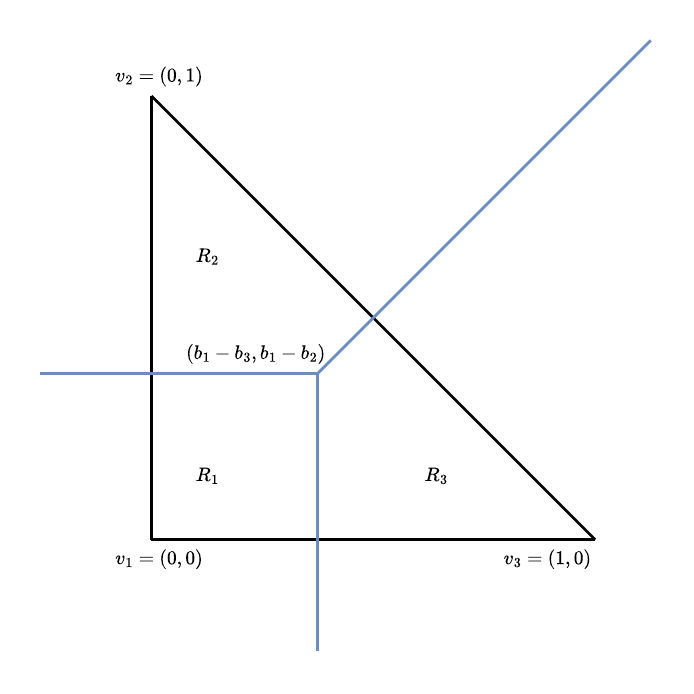}
        \caption[Structure of $C_0$ and $R_1,R_2,R_3$]{Structure of $C_0$ and the sets $R_1,R_2,R_3$. Each set contains its boundary (shown in blue) and the intersection point is determined by $b$.}
        \label{fig : lbcm structure}
    \end{figure}
    With these sets the sub-differential of $v(x;b)$ over all of $\mathbb{R}^2$ can be explicitly written as
    \begin{equation*}
        \partial v(x;b) = \begin{cases}
            \{(0,0)\} & x \in \text{int}(R_1)\\
            \{(0,1)\} & x \in \text{int}(R_2)\\
            \{(1,0)\} & x \in \text{int}(R_3) \\
            \normalfont{\text{conv}}(\{(0,0),(0,1)\}) & x \in (R_1 \cap R_2) \setminus R_3 \\
            \normalfont{\text{conv}}(\{(0,0),(1,0)\}) & x \in (R_1 \cap R_3) \setminus R_2 \\
            \normalfont{\text{conv}}(\{(0,1),(1,0)\}) & x \in (R_2 \cap R_3) \setminus R_1 \\
            \normalfont{\text{conv}}(\{(0,0),(0,1),(1,0)\}) & x \in R_1 \cap R_2 \cap R_3
        \end{cases}
    \end{equation*}
    The set $R_1 \cap R_2 \cap R_3$ consists of the single point $\{(b_1 - b_3, b_1 - b_2)\}$. If $b_1 - b_3 \in [0,1)$ then $a_1 = b_1 - b_3$ is the single value where it is possible that there are $0 \leq a_2' < a_2 \leq 1 - (b_1 - b_3)$
    \begin{equation*}
        \langle e_1, T((a_1,a_2))\rangle \geq \langle e_1, T((a_1,a_2')) \rangle.
    \end{equation*}
    Otherwise if $a_1 < b_1 - b_3$ and $0 \leq a_2 < a_2' \leq 1 - a_1$ then $(a_1,a_2), (a_1,a_2') \notin R_3$ and there are five possible configurations:
    \begin{enumerate}
        \item $(a_1,a_2) \in \text{int}(R_1)$ and $(a_1,a_2') \in \text{int}(R_1)$,
        \item $(a_1,a_2) \in \text{int}(R_1)$ and $(a_1,a_2') \in R_1 \cap R_2$,
        \item $(a_1,a_2) \in \text{int}(R_1)$ and $(a_1,a_2') \in \text{int}(R_2)$,
        \item $(a_1,a_2) \in R_1 \cap R_2$ and $(a_1,a_2') \in \text{int}(R_2)$,
        \item $(a_1,a_2) \in \text{int}(R_2)$ and $(a_1,a_2') \in \text{int}(R_2)$.
    \end{enumerate}
    One can check explicitly using the formula for the sub-differential that in each of these cases that $\langle e_1, T((a_1,a_2)) \rangle = \langle e_1, T((a_1,a_2')) \rangle = 0.$
    If $a_1 > b_1 - b_3$ then $(a_1,a_2), (a_1,a_2') \notin R_1$ and there are again five cases.
    \begin{enumerate}
        \item $(a_1,a_2) \in \text{int}(R_3)$ and $(a_1,a_2') \in \text{int}(R_3)$
        \item $(a_1,a_2) \in \text{int}(R_3)$ and $(a_1,a_2') \in R_2 \cap R_3$
        \item $(a_1,a_2) \in \text{int}(R_3)$ and $(a_1,a_2') \in \text{int}(R_2)$
        \item $(a_1,a_2) \in R_2 \cap R_3$ and $(a_1,a_2') \in \text{int}(R_2)$
        \item $(a_1,a_2) \in \text{int}(R_2)$ and $(a_1,a_2') \in \text{int}(R_2)$
    \end{enumerate}
    In the first three cases 
    \begin{equation*}
        \langle e_1, T((a_1,a_2)) \rangle = 1 \geq \langle e_1, T((a_1,a_2')) \rangle,
    \end{equation*}
    in the fourth 
    \begin{equation*}
        \langle e_1, T((a_1,a_2)) \rangle \geq 0 = \langle e_1, T((a_1,a_2')) \rangle,
    \end{equation*}
    and in the fifth
    \begin{equation*}
        \langle e_1, T((a_1,a_2)) \rangle = 0 = \langle e_1, T((a_1,a_2')) \rangle.
    \end{equation*}
    This completes the proof.
\end{proof}

Lemma \ref{lem : first coordinate decrease} and the following lemma are the two main tools used to show that the function $\phi_0$ is a counterexample.
\begin{lemma} \label{lem : phi issues}
    For $x \in [0,1]$ and $0 \leq y_1 < y_2 \leq 1 - x$ it holds that 
    \begin{equation*}
        \frac{x}{8}(y_2 - y_1) \leq \langle e_1, \nabla \phi_0(x,y_2) - \nabla \phi_0(x,y_1) \rangle.
    \end{equation*}
\end{lemma}
Before proceeding to the proof we remark that the maps in Lemma \ref{lem : first coordinate decrease} behave in an \textit{opposite way} to $\nabla \phi_0$. Specifically, fix a map $T$ as in   Lemma \ref{lem : first coordinate decrease} and let $x \in [0,1]$ take any value besides $x = 1$ and the single point of issue in Lemma \ref{lem : first coordinate decrease} (if it exists). Along the segment $\{(x,y) \ | 0 \leq y \leq 1- x\}$ the inner product with $e_1$ \textit{decreases} for $T$ as $y$ increases, while the inner product of $e_1$ and  $\nabla \phi_0$ \textit{increases} as $y$ increases. These properties will be leveraged in the proof of Theorem \ref{thm : questions negative} below.
\begin{proof}
    Fix an $x \in [0,1]$ and a $y \in [0,1-x]$ and compute the partial derivative
    \begin{equation*}
        \frac{\partial}{\partial y} \langle e_1, \nabla \phi_0(x,y) \rangle = \frac{x}{2(2-y)^2}.
    \end{equation*}
    Since $y \in [0,1]$ we have
    \begin{equation*}
        \frac{x}{8} \leq \frac{x}{2(2-y)^2}.
    \end{equation*}
    From this it follows
    \begin{align*}
        \langle e_1, \nabla \phi_0(x,y_2) \rangle &= \langle e_1, \nabla \phi_0(x,y_1) \rangle + \int_{y_1}^{y_2} \frac{\partial}{\partial y} \langle e_1, \nabla \phi_0(x,y) \rangle dy \\
        &\geq \langle e_1, \nabla \phi_0(x,y_1) \rangle + \int_{y_1}^{y_2} \frac{\partial}{\partial y} \frac{x}{8} dy \\
        &= \langle e_1, \nabla \phi_0(x,y_1) \rangle + \frac{x}{8}(y_2 - y_1).
    \end{align*}
    Re-arranging proves the bound.
\end{proof}

\begin{figure}[t!]
    \centering
    \includegraphics[width=1.0\linewidth]{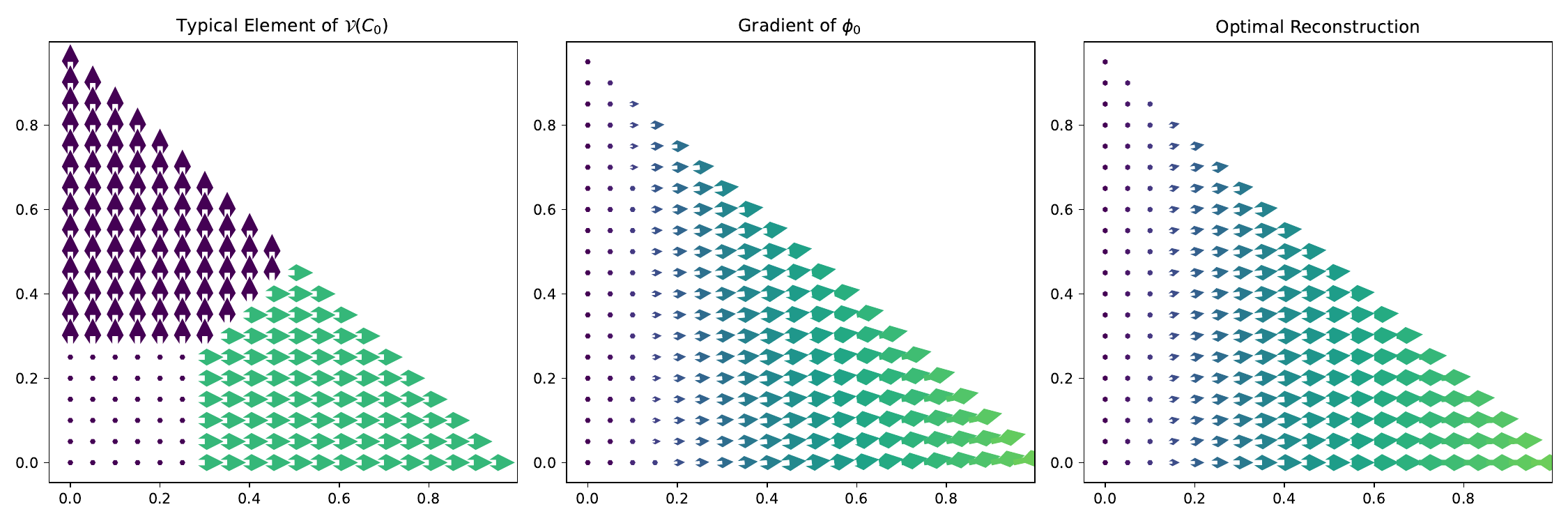}
    \caption[Counter example in Theorem \ref{thm : questions negative}]{Counter example in Theorem \ref{thm : questions negative}. On the left we have a typical element of $\mathcal{V}(C_0)$, in the middle the gradient of $\phi_0$, and on the right a numerically obtained optimal reconstruction of $\phi_0$ using a convex combination of elements in $\mathcal{V}(C_0)$ represented as vector fields. Arrow color corresponds to the magnitude of the first coordinate.}
    \label{fig : counter example}
\end{figure}

The essence of Lemmas \ref{lem : first coordinate decrease} and \ref{lem : phi issues} are visualized in Figure \ref{fig : counter example}. Along any vertical strip, when moving from bottom to top, the first component cannot increase for an element of $\mathcal{V}(C_0)$ while for $\nabla \phi_0$ it is strictly increasing. The optimal reconstruction fails to capture this behavior and has constant first component along any vertical line. The proof of Theorem \ref{thm : questions negative} makes this precise and gives a lower bound on the amount of error. Before showing it we must establish a technical lemma which will dispose of many edge cases in the proof of Theorem \ref{thm : questions negative}.

\begin{lemma} \label{lem : almost every}
    Let $T \in \normalfont{\text{conv}}(\mathcal{V}(C_0))$. Then for almost every $a_1 \in [0,1]$ it holds that for all $0\leq a_2 < a_2' \leq 1-a_1$ we have 
    \begin{equation*}
        \langle e_1, T((a_1,a_2))\rangle \geq \langle e_1, T((a_1,a_2')) \rangle.
    \end{equation*}
\end{lemma}
\begin{proof}
    Let $T \in \normalfont{\text{conv}}(\mathcal{V}(C_0))$ and let $\lambda \in \mathcal{P}(\mathcal{V}(C_0))$ be such that 
    \begin{equation*}
        T = \int_{\mathcal{V}(C_0)} T_i d\lambda(T_i).
    \end{equation*}
    The choice of $T_i$ can be interpreted as the optimal transport map to the corresponding measure $\mu_i$, or simply as the integrated transport map.
    
    Define the set
    \begin{equation*}
        A = \{ a_1 \in [0,1] \ | \ \exists \ 0 \leq a_2 < a_2' \leq 1 - a_1 \text{ s.t. } \langle e_1, T((a_1,a_2))\rangle < \langle e_1, T((a_1,a_2')) \rangle \}
    \end{equation*}
    which corresponds to the set where the property in Lemma \ref{lem : first coordinate decrease} fails. Now for $a_1 \in A$ define the set 
    \begin{equation*}
        \mathcal{T}_{a_1} = \{ T' \in \mathcal{V}(C_0) \ | \ \exists \ 0 \leq a_2 < a_2' \leq 1 - a_1 \text{ s.t. } \langle e_1, T'((a_1,a_2))\rangle < \langle e_1, T'((a_1,a_2')) \rangle \}.
    \end{equation*}
    This corresponds to the set of maps where the property in Lemma \ref{lem : first coordinate decrease} fails at the point $a_1$. 

    Observe that by Lemma \ref{lem : first coordinate decrease} it must be the case that for $a_1 \neq a_1'$ that
    \begin{equation*}
        \mathcal{T}_{a_1} \cap \mathcal{T}_{a_1'} = \varnothing
    \end{equation*}
    since each function $T' \in \mathcal{V}(C_0)$ fails to have the property in Lemma \ref{lem : first coordinate decrease} at at most a single point.

    Now let $a_1 \in A$ and let $0 \leq a_2 < a_2' \leq 1 - a_1$ be such that $\langle e_1, T((a_1,a_2))\rangle < \langle e_1, T((a_1,a_2')) \rangle$ which implies $\langle e_1, T((a_1,a_2))\rangle - \langle e_1, T((a_1,a_2')) \rangle < 0$. It holds that
    \begin{align*}
        0 &> T((a_1,a_2))\rangle - \langle e_1, T((a_1,a_2')) \rangle \\
        &=\int_{\mathcal{V}(C_0)} \langle e_1, T_i((a_1,a_2))\rangle - \langle e_1, T_i((a_1,a_2')) \rangle d\lambda(T_i) \\
        &=\int_{\mathcal{T}_{a_1}} \langle e_1, T_i((a_1,a_2))\rangle - \langle e_1, T_i((a_1,a_2')) \rangle d\lambda(T_i) + \int_{\mathcal{V}(C_0) \setminus \mathcal{T}_{a_1}} \langle e_1, T_i((a_1,a_2))\rangle - \langle e_1, T_i((a_1,a_2')) \rangle d\lambda(T_i) \\
        &\geq \int_{\mathcal{T}_{a_1}} \langle e_1, T_i((a_1,a_2))\rangle - \langle e_1, T_i((a_1,a_2')) \rangle d\lambda(T_i) \\
        &\geq \int_{\mathcal{T}_{a_1}} -1 d\lambda(T_i) = -\lambda[\mathcal{T}_{a_1}]
    \end{align*}
    which implies that $\lambda[\mathcal{T}_{a_1}] > 0$. In particular this holds for every $a_1 \in A$. Since the sets $\mathcal{T}_{a_1}$ are disjoint and for every $a_1 \in A$ the set $\mathcal{T}_{a_1}$ receives strictly positive mass from $\lambda$ it must be the case that $A$ is countable since for any probability measure any disjoint family of sets with positive probability is at most countable. 

    In particular since countable sets have measure 0 we can conclude that the set $A$ has measure 0 with respect to the Lebesgue measure on $[0,1]$. Therefore almost every $a_1 \in [0,1]$ is not contained in the set $A$ and therefore has the claimed property.
\end{proof}

\begin{proof}[\textbf{Proof of Theorem \ref{thm : questions negative}}]
    Let $T \in \normalfont{\text{conv}}(\mathcal{V}(C_0))$. By Lemma \ref{lem : almost every} we have for almost every $x$ the function $y \mapsto \left \langle e_1, T(x,y) \right \rangle$ is monotonically decreasing. Let $S$ denote the set of measure zero where this fails. Since $S$ has measure 0 we can ignore it while integrating. Now let $x \in [0,1) \setminus S$. Define the point 
    \begin{equation*}
        y_x = \sup \left \{ y \in [0,1-x] \ | \ \langle e_1, \nabla \phi_0(x,y) \rangle \leq  \left \langle e_1, T(x,y) \right \rangle \right \}.
    \end{equation*}
    with the convention that $y_x = 0$ if the supremum is over an empty set.
    Along the segment $\{(x,y) \ | \ y \in [0, 1 - x] \}$ we have
    \begin{align*}
        &\int_0^{1-x} \norm{T(x,y) - \nabla \phi_0(x,y)}_2 dy \\
        \geq &\int_0^{1-x} \left | \left \langle e_1, T(x,y) - \nabla \phi_0(x,y) \right \rangle  \right | dy \\
        =& \int_0^{y_x} \left | \left \langle e_1, T(x,y) - \nabla \phi_0(x,y) \right \rangle  \right | dy \\
        & + \int_{y_x}^{1-x} \left | \left \langle e_1, T(x,y) - \nabla \phi_0(x,y) \right \rangle  \right | dy \\
        =& \int_0^{y_x} \left \langle e_1, T(x,y) - \nabla \phi_0(x,y) \right \rangle   dy \\
        & + \int_{y_x}^{1-x}  \left \langle e_1, \nabla \phi_0(x,y) - T(x,y)  \right \rangle  dy \\
        \geq &\int_0^{y_x} \left \langle e_1, \nabla \phi_0(x,y_x) - \nabla \phi_0(x,y) \right \rangle  dy + \int_{y_x}^{1-x} \left \langle e_1, \nabla\phi_0(x,y) - \nabla \phi_0(x,y_x) \right \rangle  dy \\
        \geq &\int_0^{y_x} \frac{x}{8} (y_x - y) dy + \int_{y_x}^{1-x} \frac{x}{8} (y - y_x) dy \\
        = &\frac{x}{8} \left [ \left ( y_x^2 - \frac{y_x^2}{2} \right ) + \left ( \frac{(1-x)^2}{2} - \frac{y_x^2}{2} - y_x(1-x) + y_x^2 \right ) \right ]\\
        = &\frac{x}{8} \left [ \frac{(1-x)^2}{2} + y_x^2 - y_x(1-x) \right ] \geq \frac{x}{8} \cdot \frac{(1-x)^2}{4}.
    \end{align*}
    The first inequality is a consequence of $\norm{z}_2^2 = \langle e_1,z\rangle^2 + \langle e_2, z\rangle^2 \geq \langle e_1,z\rangle^2.$  For the second inequality we have that since $x \in S$ the function $T(x,\cdot)$ is monotonically decreasing and by continuity of $\nabla \phi_0(x,\cdot)$ (at least on $C_0$) that for all $0 \leq y < y_x$
    \begin{align*}
        \left \langle e_1, T(x,y) \right \rangle &\geq \lim_{y_0^+ \rightarrow y_x} \left \langle e_1, T(x,y_0) \right \rangle \\
        &\geq \lim_{y_0^+ \rightarrow y_x} \left \langle e_1, \nabla \phi_0(x,y_0) \right \rangle = \left \langle e_1, \nabla \phi_0(x,y_x) \right \rangle
    \end{align*}
    which implies
    \begin{equation*}
        \int_0^{y_x} \left \langle e_1, T(x,y) - \nabla \phi_0(x,y) \right \rangle   dy \geq \int_0^{y_x} \left \langle e_1, \nabla \phi_0(x,y_x) - \nabla \phi_0(x,y) \right \rangle  dy.
    \end{equation*}
    Similarly for the other integral we have for $y_x < y \leq 1-x$
    \begin{align*}
        \left \langle e_1, T(x,y) \right \rangle 
        &\leq \lim_{y_0^- \rightarrow y_x} \left \langle e_1, T(x,y_0) \right \rangle \\
        &\leq \lim_{y_0^- \rightarrow y_x} \left \langle e_1, \nabla \phi_0(x,y_0) \right \rangle = \left \langle e_1, \nabla \phi_0(x,y_x) \right \rangle
    \end{align*}
    which implies
    \begin{equation*}
        \int_{y_x}^{1-x}  \left \langle e_1, \nabla \phi_0(x,y) - T(x,y)  \right \rangle  dy \geq \int_{y_x}^{1-x} \left \langle e_1, \nabla\phi_0(x,y) - \nabla \phi_0(x,y_x) \right \rangle  dy.
    \end{equation*}
    The third inequality uses Lemma \ref{lem : phi issues}. The fourth uses that $y_x \in [0,1-x]$ and minimizes the bound.
    
    Using this inequality and the fact that $S$ is negligible we can compute
    \begin{align*}
        &\int \norm{T - \nabla \phi_0}_2 dU(C_0) \\
        = &2\int_0^1 \int_0^{1-x} \norm{T - \nabla \phi_0(x,y)}_2 dy dx \\
        \geq &2\int_0^1 \frac{x(1-x)^2}{32} dx = \frac{1}{192}.
    \end{align*}
    This bound holds uniformly over $\mathcal{V}(C_0)$ and disproves Question \ref{que : convex maps}. As noted above Question \ref{que : is dense} is equivalent to Question \ref{que : convex maps}.
\end{proof}

\section{Details for Versions of Algorithm \ref{alg : covariance estimate}}
\label{sec:Algorithm3}

We consider three versions of Algorithm \ref{alg : covariance estimate}.
\begin{enumerate}
    \item (\textbf{BCM}) $\mathtt{EstimateCoordinate}$ is done using the $W_{2}$BCM with $\eta = N(0,\hat{\Sigma}_{\text{emp}})$ and $\mu_i = \mathcal{N}(0,\Sigma_i)$. Next construct $A_{ij}$ as described in Section \ref{sec:SynthesisAnalysis} and select $\hat{\lambda}_{\text{bcm}}$ as a minimizer in \eqref{eqn:BCM_Analysis_QP}. $\mathtt{EstimateCovariance}$ is done using Algorithm 1 in \cite{chewi2020gradient} with coordinate $\hat{\lambda}_{\text{bcm}}$ and $\Sigma_1,...,\Sigma_m$.
    
    \item (\textbf{LBCM}) Fix a $\mu_0 = \mathcal{N}(0,\Sigma_0)$, set $\eta = \mathcal{N}(0,\hat{\Sigma}_{\text{emp}})$, and set $\mu_i = \mathcal{N}(0,\Sigma_i)$. Compute the transport maps $T_{\mu_0}^{\mu_i}$ (which are of the form $T_{\mu_0}^{\mu_i}(x) = C_i(x)$) and $T_{\mu_0}^\eta$ using \begin{equation} \label{eq : map matrix}
    C_i = S_\eta^{-1/2}\left ( S_\eta^{1/2} S_i S_\eta^{1/2} \right )^{1/2} S_\eta^{-1/2}.
\end{equation} Select $\hat{\lambda}_{\text{lbcm}}$ as a minimizer in \eqref{eqn:AnalysisLBCM}. $\mathtt{EstimtateCovariance}$ is done using the estimate
    \begin{equation*}
        \left ( \sum_{i=1}^m (\hat{\lambda}_{lbcm})_i C_i  \right )^T \Sigma_0 \left ( \sum_{i=1}^m (\hat{\lambda}_{lbcm})_i C_i  \right )
    \end{equation*}
    which corresponds to the covariance of $\left (  \sum_{i=1}^m (\hat{\lambda}_{lbcm})_i T_{\mu_0}^{\mu_i} \right )\#\mu_0$.
    
    \item (\textbf{Maximum Likelihood Estimation (MLE)}) $\mathtt{EstimateCoordinate}$ is done by setting $\eta = N(0,\hat{\Sigma}_{emp})$ and returning $\hat{\lambda}_{\text{mle}}$ as the minimizer of $\min_{\lambda \in \Delta^m} D_{KL}(\eta \ || \ \nu_\lambda)$ which can be numerically approximated. $\mathtt{EstimateCovariance}$ is done using using Algorithm 1 in \cite{chewi2020gradient} with coordinate $\hat{\lambda}_{\text{mle}}$ and $\Sigma_1,...,\Sigma_m$.  We defer details of how the MLE strategy is implemented to Appendix \ref{Appendix:MLE_Details}. 
\end{enumerate}

\section{Implementation Details for MLE} \label{Appendix:MLE_Details}

In order to solve the maximum likelihood estimation problem we differentiate through a truncated version of \cite{chewi2020gradient} Algorithm 1 and perform projected gradient descent. This is implemented using the auto-differentiation library PyTorch. The procedure is summarized in Algorithm \ref{alg:MLE}.

\begin{algorithm}[t!]
    \caption{MLE}\label{alg:MLE}
    \begin{algorithmic}[1]
    \STATE {\bfseries Input:} $\{S_i\}_{i=0}^m$, $\eta > 0$, MaxIters > 0, FPIters > 0, SQIters > 0 
    \STATE $i \leftarrow 0, \lambda  \leftarrow (1/m)\bm{1}$
    \WHILE{Not Converged and $i < $ MaxIters}
        \STATE $i \leftarrow i + 1$ 
        \STATE $\nabla \mathcal{L}(\lambda) \leftarrow $ $\mathtt{BackPropLoss}(\{S_j\}_{j=0}^m$, $\lambda$, FPIters, SQIters)
        \STATE $\lambda \leftarrow \mathtt{SimplexProject}(\lambda - \eta \nabla \mathcal{L}(\lambda))$
    \ENDWHILE
    \STATE \textbf{Return} $\lambda$
    \end{algorithmic}
\end{algorithm}
Here $\mathtt{SimplexProject}(x)$ is defined as 
\begin{equation*}
    \mathtt{SimplexProject}(x) = \argmin_{y \in \Delta^m} \norm{x - y}_2.
\end{equation*}
which enforces the constraints on $\lambda$ at each iteration.

To compute $\nabla \mathcal{L}(\lambda)$ we use auto-differentiation to obtain a gradient of the procedure given in Algorithm \ref{alg:cl}. In this procedure the square root of a matrix is computed using SQIters number of Newton-Schulz iterations. The forward pass of the loss computation is given in Algorithm \ref{alg:cl}, and the gradient is obtained by back propagation through it.
\begin{algorithm}[htbp!]
\caption{ComputeLoss} \label{alg:cl}
\begin{algorithmic}[1]
\STATE {\bfseries Input:} $\{S_i\}_{i=0}^m$, $\lambda$, FPIters > 0, SQIters > 0
\STATE BC $\leftarrow S_0$ 
\FOR{$j = 1,..., $ FPIters}
    \STATE BC\_root $\leftarrow \mathtt{SquareRoot}$(BC, SQIters) 
    \STATE  BC\_root\_inv $\leftarrow \mathtt{Invert}$(BC\_root) 
    \STATE BC $\leftarrow$ BC\_root\_inv $ \left ( \sum_{i=1}^m \lambda_i \mathtt{SquareRoot}\text{(BC\_root, SQIters,} S_i \text{BC\_root} ) \right )$ BC\_root\_inv 
\ENDFOR
\STATE \textbf{Return} Tr($\text{BC}^{-1}S_0$) + log det BC 
\end{algorithmic}
\end{algorithm}

The parameters that we chose in our experiments are given in Table \ref{tab:ag_params}.
\begin{table}[b]
\centering
\begin{tabular}{|c|c|}
\hline
Parameter & Value  \\ \hline
$\eta$    & 0.0003 \\ \hline
MaxIters  & 500    \\ \hline
FPIters   & 10     \\ \hline
SQIters   & 10     \\ \hline
\end{tabular}
\caption{\label{tab:ag_params}Parameters used when performing the maximum likelihood estimation.}
\end{table}
These parameters were sufficient to ensure that both the matrix square roots and fixed point iterations converged, and $\lambda$ always converged before the final iteration was reached.

\section{Implementation Details for 2-Wasserstein Barycenter Synthesis}\label{subsect:WGD details}

\begin{algorithm}
\caption{Iterative 2-Wasserstein Barycenter Synthesis}\label{alg:WGD}
    \begin{algorithmic}[1]
        \STATE {\bfseries Input:} Initial probability measure $\rho_0$, step-size $\alpha>0$, iterations $k\geq 1$.
        \FOR{$1\leq \ell \leq k$}
            \FOR{$1\leq j \leq m$}
                \STATE Compute $\pi_{\rho_{\ell-1}}^{\mu_j}=\argmin_{\pi\in \Pi(\rho_{\ell-1},\mu_j)} \int \frac{1}{2}\|x-y\|^2 d\pi(x,y)$.
                \STATE Define map $\bar{T}^{\mu_j}_{\rho_{\ell-1}}(x)\triangleq\mathbb{E}_{(X,Y)\sim \pi_{\rho_{\ell-1}}^{\mu_j}}[Y|X=x]$
            \ENDFOR
            \STATE $\rho_\ell=[(1-\alpha)Id+\alpha\sum_{j=1}^m\lambda_j\bar{T}^{\mu_j}_{\rho_{\ell-1}}]_{\#}[\rho_{\ell-1}]$
        \ENDFOR
        \STATE \textbf{Return} $\rho_k$
    \end{algorithmic}
\end{algorithm}

We synthesize an approximate 2-Wasserstein barycenter using Algorithm \ref{alg:WGD}. When an optimal transport map exists between iterates $\rho_{\ell}$ and references $\mu_j$, Algorithm \ref{alg:WGD} agrees with the algorithm implemented in \cite{zemel2019frechet}, where they establish convergence under regularity assumptions on the target measures. Algorithm \ref{alg:WGD} can also be viewed as a forward Euler discretization of the Wasserstein gradient flow for the functional $\sum_{j=1}^m\lambda_j W^2_2(-,\mu_j)$.

\section{Image Algorithms}

\begin{algorithm}[htbp!]
    \caption{Image to Measure} \label{alg : image to measure}
    \begin{algorithmic}[1]
        \STATE {\bfseries Input:} Matrix representation of an image $\mathtt{I} \in \mathbb{R}_+^{d \times d}$
        \STATE Set $s = \sum_{i,j=1}^d \mathtt{I}_{i,j}$
        \STATE \textbf{Return} $\frac{1}{s} \sum_{i,j=1}^d \mathtt{I}_{i,j} \delta_{(i/d,j/d)}$
    \end{algorithmic}
\end{algorithm}

\begin{algorithm}[htbp!]
    \caption{Measure to Image} \label{alg : measure to image}
    \begin{algorithmic}[1]
        \STATE {\bfseries Input:} Support locations $S \in \mathbb{R}^{n\times 2}$, Masses $\{b_{k}\} \in\mathbb{R}_+^n$, Output size $d$, Resolution $r$, Bandwidth $b$, Lower bound $\ell$.
        \STATE Create $\mathtt{K} \in \mathbb{R}^{rd \times rd}_+$ with $\mathtt{K}_{ij} = \sum_{k=1}^n b_k \exp\left ( -\frac{\norm{(i/rd,j/rd) - S_k}_2^2}{b^2} \right )$
        \STATE Set $\mathtt{K} = \left ( \sum_{i,j=1}^{rd} \mathtt{K}_{ij} \right )^{-1}\cdot\mathtt{K}$
        \STATE For each $i,j$ Set $\mathtt{K}_{ij} = \mathtt{K}_{ij} \cdot \pmb{1}[\mathtt{K}_{ij} > \ell]$ 
        \STATE Create $\mathtt{I} \in \mathbb{R}^{d\times d}$ with $\mathtt{I}_{ij} = \sum_{k,l=1}^r \mathtt{K}_{(i-1)*r+k, (l-1)*r+l}$
        \STATE \textbf{Return:} $\left ( \sum_{i,j=1}^d \mathtt{I}_{ij }\right )^{-1}\cdot\mathtt{I}$
    \end{algorithmic}
\end{algorithm}

\section{Additional Experiments}

\label{sec:AdditionalExperiments}

To conclude we briefly consider the impact that is had by the choice of the base measure in the LBCM in the problem above. We consider four choices for this measure which are presented in Figure \ref{fig : base measures}.  The first measure is the convolutional barycenter of the 10 reference digits using the method of \cite{solomon2015convolutional}. The second is a uniform image, the third uses two squares in the corners, the fourth is supported on a circle, and the final is supported on squares supported in the four corners of the image. The reconstruction of the same digits using each of these as the base measure is shown in Figure \ref{fig : lbcm reconstructions}.

\begin{figure}[t!]
    \centering
    \includegraphics[angle=-90, width=0.99\linewidth]{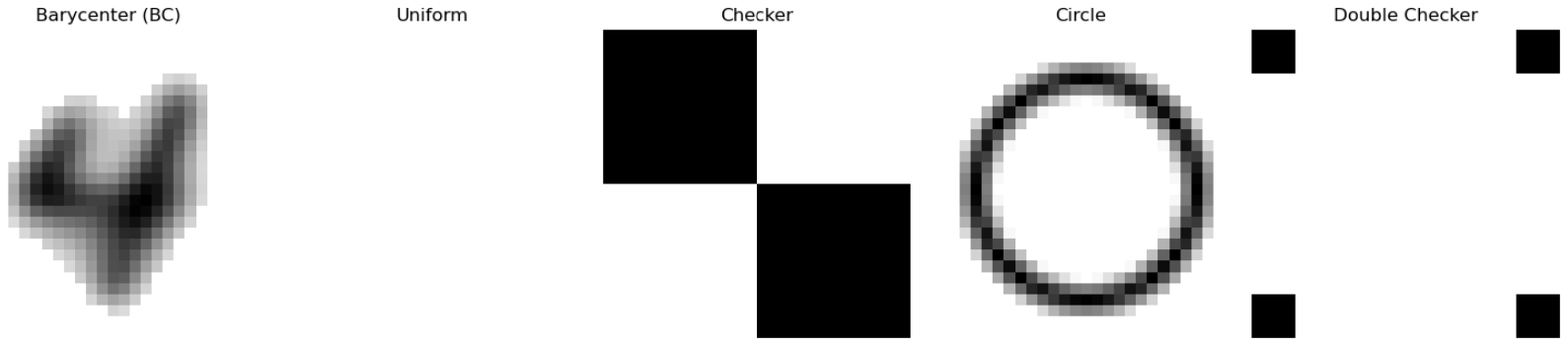}
    \caption{Base measures used in Figure \ref{fig : lbcm reconstructions}}
    \label{fig : base measures}
\end{figure}

\begin{figure}[t!]
    \centering
    \includegraphics[width=0.99\linewidth]{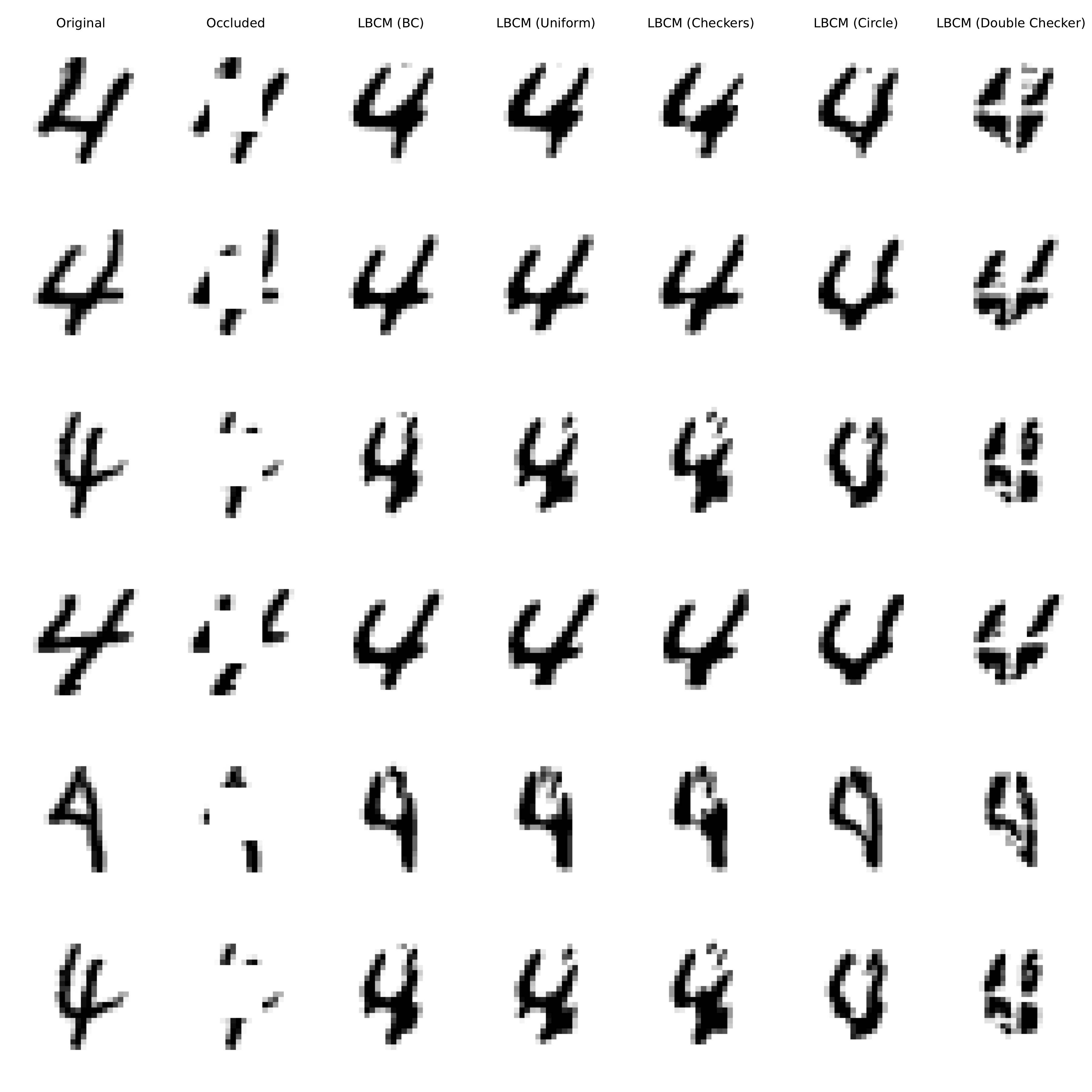}
    \caption{Reconstruction of occluded digits using the LBCM with four different base measures}
    \label{fig : lbcm reconstructions}
\end{figure}

In general the barycenter digit tends to give visually the best reconstruction while the uniform and checkered base measures perform comparably to each other and the circle yields the worst reconstructions. This highlights the fact that the choice of the base measure is meaningful and can significantly impact the recovery. The question of which base measures lead to the best performance is left to future work.

\end{document}